\documentclass{article} 

\usepackage{microtype}
\usepackage{graphicx}
\usepackage{subfigure}
\usepackage{booktabs} 

\usepackage{hyperref}

\usepackage[accepted]{icml2024}
\usepackage{times}

\usepackage{amsmath,amsfonts,bm}

\def\1{\bm{1}}

\DeclareMathAlphabet{\mathsfit}{\encodingdefault}{\sfdefault}{m}{sl}
\SetMathAlphabet{\mathsfit}{bold}{\encodingdefault}{\sfdefault}{bx}{n}

\newcommand{\R}{\mathbb{R}}

\usepackage{url}
\usepackage{amsmath}
\usepackage{amssymb}
\usepackage{mathtools}
\usepackage{amsthm}
\usepackage{mathrsfs}
\usepackage{stfloats}
\usepackage{comment}
\usepackage{cleveref}
\usepackage{enumerate}
\usepackage{mathrsfs}

\newtheorem{definition}{Definition}[section]
\newtheorem{theorem}{Theorem}[section]

\newtheorem{lemma}{Lemma}[section]

\theoremstyle{definition}

\theoremstyle{remark}

\numberwithin{equation}{section}

\Crefformat{equation}{eq.(#2#1#3)}

\icmltitlerunning{HERTA: A High-Efficiency and Rigorous Training Algorithm for Unfolded Graph Neural Networks}

\begin{document}

\twocolumn[
\icmltitle{HERTA: A High-Efficiency and Rigorous Training Algorithm\\ for Unfolded Graph Neural Networks}




\begin{icmlauthorlist}
\icmlauthor{Yongyi Yang}{umich}
\icmlauthor{Jiaming Yang}{umich}
\icmlauthor{Wei Hu}{umich}
\icmlauthor{Michał Dereziński}{umich}
\end{icmlauthorlist}

\icmlaffiliation{umich}{Department of Computer Science and Engineering at Michigan
, University of Michigan}

\icmlcorrespondingauthor{Michał Dereziński}{derezin@umich.edu}

\icmlkeywords{Graph Neural Networks, Spectral Sparsificaton, Optimization, Matrix Sketching}

\vskip 0.3in
]
\printAffiliationsAndNotice{} 

\begin{abstract}
    As a variant of Graph Neural Networks (GNNs), Unfolded GNNs offer enhanced interpretability and flexibility over traditional designs.
    Nevertheless, they still suffer from scalability challenges when it comes to the training cost. Although many methods have been proposed to address the scalability issues, they mostly focus on per-iteration efficiency, without worst-case convergence guarantees. Moreover, those methods typically add components to or modify the original model, thus possibly breaking the interpretability of Unfolded GNNs. In this paper, we propose HERTA: a High-Efficiency and Rigorous Training Algorithm for Unfolded GNNs that accelerates the whole training process, achieving a nearly-linear time worst-case training guarantee. Crucially, HERTA converges to the optimum of the original model, thus preserving the interpretability of Unfolded GNNs. Additionally, as a byproduct of HERTA, we propose a new spectral sparsification method applicable to normalized and regularized graph Laplacians that ensures tighter bounds for our algorithm than existing spectral sparsifiers do. Experiments on real-world datasets verify the superiority of HERTA as well as its adaptability to various loss functions and optimizers.
\end{abstract}

\section{Introduction}\label{sec:introduction}

Graph Neural Networks (GNNs) have become a powerful modern tool for handling graph data because of their strong representational capabilities and ability to explore the relationships between data points \cite{gcn,gnn-intro,GIN}. Like many deep learning models, GNNs are generally designed by heuristics and experience, which makes analyzing and understanding them a difficult task. \emph{Unfolded GNNs} \cite{ivsu} are a type of GNNs that are rigorously derived from an optimization problem, which makes their inductive bias and training dynamics easier to interpret. 

Despite their enhanced interpretability, training unfolded models can be extremely expensive, especially when the graph is large and (relatively) dense. This issue can come from two aspects: \begin{enumerate}[1)]
    \item \textit{Slow iterations}: in every iteration the whole graph must be fed into the model, and the interactive nature of graphs prevents trivially utilizing techniques like mini-batch training;\label{stat:slow-iterations}
    \item \textit{Slow convergence}: the training convergence rate is related to the connectivity of the graph and the well-conditionedness of the node features, and the model converges slowly when the data is ill-conditioned.\label{stat:many-iterations}
\end{enumerate}

Many methods have been proposed to address the high cost of training unfolded models \cite{lesongstochastic,clustergcn,decoupling,vqgnn,gnnautoscale,musegnn}, primarily by using graph sampling to enable mini-batch training schemes that reduce the per-iteration cost  (Issue \ref{stat:slow-iterations}).
However, these approaches typically require distorting the underlying optimization objective explicitly or implicitly, thus diminishing the rigorous and interpretable nature of Unfolded GNNs. Moreover, existing works have mostly focused on per-iteration efficiency, while the convergence rate of the optimization in training Unfolded GNNs (Issue \ref{stat:many-iterations}) remains un-addressed, leading to methods that are not robust to ill-conditioned data. 


We propose HERTA: a High-Efficiency and Rigorous Training Algorithm for Unfolded GNNs, which is an algorithmic framework designed to address both Issues \ref{stat:slow-iterations} and \ref{stat:many-iterations}, while at the same time preserving the rigorous and interpretable nature of Unfolded GNNs.
Unlike many existing methods that require changing the GNN model or defining a new objective, HERTA converges to the optimum of the original training problem, and thus preserves the interpretability of the unfolded models. Moreover, HERTA uses a specialized preconditioner to ensure fast linear convergence to the optimum, requiring only a logarithmic number of passes over the data. We show empirically that HERTA can be used to accelerate training for a variety of GNN objectives, as an extension of popular optimization techniques.

A key novelty of our algorithm lies in the construction of a preconditioner for the GNN objective, which accelerates the convergence rate. To construct this preconditioner, we design a spectral sparsification method for approximating the squared inverse of a normalized and regularized Laplacian matrix, by relying on an extension of the notion of effective resistance to overcome the limitations of existing graph sparsification tools. Below, we present an informal statement of our main result.

\begin{theorem}[Informal version of \Cref{thm:main}]\label{thm:main-informal}
    HERTA solves the $\lambda$-regularized Unfolded GNN objective \eqref{eq:bilevel-outer} with $n$ nodes, $m$ edges and $d$-dimensional node features to within accuracy $\epsilon$ in time $\tilde O\left( (m+nd) \left( \log \frac 1\epsilon\right)^2 + d^3\right)$ as long as the number of large eigenvalues of the graph Laplacian is $O(n/\lambda^2)$. 
\end{theorem}

In practice, the node feature dimensionality $d$ is generally much smaller than the graph size $m$, in which case, the running time of HERTA is $\tilde O\left( (m + nd) \left( \log \frac{1}{\epsilon}\right)^2\right)$. Notice that to describe a typical graph dataset with $m$ edges, $n$ nodes and $d$-dimensional node features, we need at least $O(m + nd)$ float or integer numbers. This shows the essential optimality of HERTA: its running time is of the same magnitude as reading the input data (up to logarithm factors).

The condition that the graph Laplacian has $O(n/\lambda^2)$ large eigenvalues is not strictly necessary; it is used here to simplify the time complexity (see Theorem \ref{thm:main} for the full statement). This condition is also not particularly restrictive, since in most practical settings the GNN parameter $\lambda$ is chosen as an absolute constant. At the same time, it is natural that the complexity of training a GNN depends on the number of large eigenvalues in the graph Laplacian, since this quantity can be interpreted as the effective dimension in the Laplacian energy function \eqref{eq:energy} used to define the Unfolded GNN (see more discussion in Section \ref{sec:efficient-twirls-training}).

\textbf{Outline.} The paper is organized as follows. In \Cref{sec:background} we introduce related work. In \Cref{sec:preliminaries} we introduce the mathematical notation and concepts used in this paper. We define our problem setting in \Cref{sec:problem-setting}. In \Cref{sec:efficient-twirls-training} we present and analyze our algorithm HERTA, and introduce the techniques that are used in proving the main result.
We conduct experiments on real-world datasets and show the results in \Cref{sec:experiments}. Finally, we conclude with some potential future directions in \Cref{sec:conclusion}. 

\section{Related Work}\label{sec:background}


\textbf{Unfolded Graph Neural Networks.}~ Unlike conventional GNN models, which are designed mainly by heuristics, the forward layers of Unfolded GNNs are derived by explicitly optimizing a graph-regularized target function \cite{revisiting-unfolding,ivsu,unified-view,elastic-gnn,twirls,unifying2,descent-step}. This optimization nature of Unfolded GNNs allows developing models in a more interpretable and controllable way. The presence of an explicit optimization problem allows  designers to inject desired properties into the model and better understand it. This approach has been used to overcome known GNN issues such as over-smoothing \cite{oono-oversmoothing,deeper-gnn,revisiting-oversmoothing} and sensitivity towards spurious edges \cite{adversarial,h2gcn}. 

\textbf{Efficient Graph Neural Network Training.}~ In order to address the scalability issue of GNNs, many techniques have been proposed. There are two major directions in the exploration of the solutions of this issue. The first direction is to adopt sampling methods. These methods include node-wise sampling \cite{graphsage}, layer-wise sampling \cite{fastgcn} and subgraph-wise sampling \cite{clustergcn,musegnn}. The second direction is to utilize embeddings from past iterations to enhance current iterations and to obtain more accurate representation with mini-batches and fewer forward layers \cite{gnnautoscale,vqgnn,lazygnn}.

As we indicated in \Cref{sec:introduction}, all of the  above 
methods aim at addressing Issue \ref{stat:slow-iterations} (\textit{Slow Iterations}), and none of them achieves a worst-case convergence guarantee without affecting the underlying GNN model. In contrast, HERTA addresses both Issue \ref{stat:slow-iterations} and also Issue \ref{stat:many-iterations} (\textit{Slow Convergence}), and possesses a theoretical guarantee on the training convergence rate while preserving the original target model. 

\textbf{Matrix Sketching and Subsampling.}~
Our techniques and analysis are closely related to matrix sketching and subsampling, which are primarily studied in the area of Randomized Numerical Linear Algebra (RandNLA, \citealt{wood14, randnla-book, martinsson2020randomized, murray2023randomized}). Given a large matrix, by sketching or sampling we compress it to a much smaller matrix with certain properties preserved, thus accelerating the solution by conducting expensive computations on the smaller matrix. This type of methods lead to improved randomized algorithms for tasks including low rank approximation \cite{halko2011finding, ridge-LSS, clarkson2017low}, linear system solving \cite{peng2021solving, solve-ls}, least squares regression \cite{rokhlin2008fast, meng2014lsrn} and so on \cite{cohen2021solving}. Our usage of these methods includes constructing spectral sparsifiers of graphs, obtained by edge subsampling \cite{spectral-sparifier-effective-resistance, first-spectral-sparsifier,spectral-sparsifier}, which are central in designing fast solvers for Laplacian linear systems \cite{vishnoi2013lx,sdd-solver-comb, sdd-solver}.

\section{Preliminaries}\label{sec:preliminaries}
For a vector ${\boldsymbol{x}}$, we denote its $\ell_2$ norm as $\|{\boldsymbol{x}}\|$. We use $\mathrm{diag}(\boldsymbol{x})$ to denote diagonalization of $\boldsymbol{x}$ into a matrix. For a matrix ${\boldsymbol{M}}$, we use $\|{\boldsymbol{M}}\|_{\mathcal F}$ and $\|\boldsymbol{M}\|$ to denote its Frobinius norm and operator norm. We also denote its largest and smallest singular value as $\sigma_{\max}({\boldsymbol{M}})$ and $\sigma_{\min}({\boldsymbol{M}})$ respectively, and denote its condition number as $\kappa({\boldsymbol{M}}) = \frac{\sigma_{\max}({\boldsymbol{M}})}{\sigma_{\min}({\boldsymbol{M}})}$. We use $\mathrm{nnz}(\boldsymbol{M})$ to denote the number of non-zero entries of $\boldsymbol{M}$. For two positive semidefinite (PSD) matrices ${\boldsymbol{\Sigma}}$ and $\widetilde {\boldsymbol{\Sigma}}$, if there exists $\epsilon \in (0,1)$ such that \begin{align}
(1-\epsilon) \widetilde {\boldsymbol{\Sigma}} \preceq {\boldsymbol{\Sigma}} \preceq (1+\epsilon) \widetilde{{\boldsymbol{\Sigma}}},
\end{align}
then we say ${\boldsymbol{\Sigma}} \approx_{\epsilon} \widetilde{{\boldsymbol{\Sigma}}}$, where $\preceq$ refers to Loewner order.

We use ${\boldsymbol{1}}$ to denote a vector with all entries equal to $1$, and the dimension of which is determined by the context if not specified. We use ${\boldsymbol{\delta}}_u$ to denote a vector with $u$-th entry equal to $1$ and all other entries equal to $0$. We also use $\delta_{u,v}$ to represent the $v$-th entry of ${\boldsymbol{\delta}}_u$, i.e.: 
$$\delta_{u,v} = \begin{cases} 1 & u = v \\ 0 & u \neq v \end{cases}.$$

For a function $\ell: \mathbb R^d \to \mathbb R$ that is bounded from bottom, a point ${\boldsymbol{w}}_0$ is called a solution of $\ell$ with $\epsilon$ error rate if it satisfies $\ell({\boldsymbol{w}}_0) \leq (1+\epsilon) \ell^*$ where $\ell^* = \inf_{{\boldsymbol{w}}} \ell({\boldsymbol{w}})$.

We adopt big-O notations in the time complexity analysis. Moreover, we use the notation $\tilde{O}(\cdot)$ to hide polynomial logarithmic factors of the problem size. As an example, $O(\log(n))$ can be expressed as $\tilde{O}(1)$. 

For a matrix ${\boldsymbol{\Pi}} \in \mathbb R^{s \times n}$, we call it a Gaussian sketch if its entries are i.i.d. $\mathcal N\left(0, 1/s\right)$ random variables. For a diagonal matrix ${\boldsymbol{R}} \in \mathbb R^{n \times n}$, we call it a diagonal Rademacher matrix if its entries are i.i.d. Rademacher random variables. See \Cref{sec:intro-to-math-tools} for more detailed introductions.

\section{Problem Setting} \label{sec:problem-setting}

Throughout this paper, we consider a graph $\mathcal G$ with $n$ nodes and $m$ edges. We denote its adjacency matrix by ${\boldsymbol{A}} \in \mathbb R^{n\times n}$, degree matrix by ${\boldsymbol{D}} = \mathop{ \mathrm{ diag } }\left({\boldsymbol{A}}{\boldsymbol{1}}\right) \in \mathbb R^{n \times n}$, and Laplace matrix (or Laplacian) by ${\boldsymbol{L}} = {\boldsymbol{D}} - {\boldsymbol{A}} \in \mathbb R^{n\times n}$. We also use the normalized adjacency $\hat {\boldsymbol{A}} = {\boldsymbol{D}}^{-1/2}{\boldsymbol{A}}{\boldsymbol{D}}^{-1/2}$ and normalized Laplacian $\hat {\boldsymbol{L}} = {\boldsymbol{I}} - \hat {\boldsymbol{A}}$. Notice that we assume there's a self-loop for each node to avoid $0$-degree nodes, which is commonly ensured in GNN implementations \citep{gcn,graphsage}, and therefore ${\boldsymbol{D}}^{-1/2}$ always exists. We also use ${\boldsymbol{B}} \in \mathbb R^{m \times n}$ to represent the incidence matrix of graph $\mathcal G$, that is, if the $i$-th edge in $\mathcal G$ is $(u_i,v_i)$, then the $i$-th row of ${\boldsymbol{B}}$ is ${\boldsymbol{\delta}}_{u_i} - {\boldsymbol{\delta}}_{v_i}$ (the order of $u_i$ and $v_i$ is arbitrary). It is not hard to see that ${\boldsymbol{L}} = {\boldsymbol{B}}^\top {\boldsymbol{B}}$.

We also assume that for each node $u$ in the graph, there is a feature vector ${\boldsymbol{x}}_u \in \mathbb R^d$, and a label vector\footnote{We allow the label to be a vector to make the concept more general. When labels are one-hot vectors, the task is a classification task, and when labels are scalars, it is a regression task.} ${\boldsymbol{y}}_u \in \mathbb R^{c}$ attached with it. Let ${\boldsymbol{X}} \in \mathbb R^{n \times d}$ and ${\boldsymbol{Y}} \in \mathbb R^{n \times c}$ be the stack of ${\boldsymbol{x}}_i$-s and ${\boldsymbol{y}}_i$-s.  We assume $n \geq \max\{d,c\}$ and ${\boldsymbol{X}}$ and ${\boldsymbol{Y}}$ are both column full rank.

The Unfolded GNN we consider in this paper is based on TWIRLS \citep{twirls}, which
is defined by optimizing the following objective (called the ``energy function''): \begin{equation}E({\boldsymbol{Z}}; {\boldsymbol{W}}) := \frac{\lambda }{2} \mathrm{Tr} \left({\boldsymbol{Z}}^\top \hat {\boldsymbol{L}} {\boldsymbol{Z}}\right) + \frac{1}{2} \left\|{\boldsymbol{Z}} - f({\boldsymbol{X}};{\boldsymbol{W}})\right\|_\mathcal F^2.\label{eq:energy}\end{equation}
Here ${\boldsymbol{Z}} \in \mathbb R^{n \times c}$ is the optimization variable that can be viewed as the hidden state of the model, ${\boldsymbol{W}}$ is the learnable parameter of the model, and $\lambda > 0$ is a hyper-parameter to balance the graph structure information and node feature information. Note that $f: \mathbb R^d \to \mathbb R^c$ is a function of $\boldsymbol{X}$ and is parameterized by matrix ${\boldsymbol{W}}$. 

For the model training, given a loss function $\mathcal \ell: \mathbb R^{n \times c} \times \mathbb R^{n \times c} \to \mathbb R$, the training of TWIRLS can be represented by the following bi-level optimization problem \begin{align}
{\boldsymbol{W}}^* \in \arg\min_{{\boldsymbol{W}} \in \mathbb R^{d \times c}} \ell\left[{\boldsymbol{Z}}^*({\boldsymbol{W}}); {\boldsymbol{Y}}\right] \label{eq:bilevel-outer}
\\ \text{s.t.}\quad   {\boldsymbol{Z}}^*({\boldsymbol{W}}) \in \arg\min_{{\boldsymbol{Z}} \in \mathbb R^{n \times c}} E({\boldsymbol{Z}};{\boldsymbol{W}}), \label{eq:bilevel-inner}
\end{align}
where the solution of inner problem \cref{eq:bilevel-inner} (i.e. ${\boldsymbol{Z}}^*({\boldsymbol{W}})$) is viewed as the model output. The outer problem \cref{eq:bilevel-outer} is viewed as the problem of training.

In \citep{twirls}, the forward process of TWIRLS is obtained by unfolding the gradient descent algorithm for minimizing $E$ given a fixed ${\boldsymbol{W}}$: \begin{align}
{\boldsymbol{Z}}^{(t+1)} & = {\boldsymbol{Z}}^{(t)} - \alpha  \nabla E\left({\boldsymbol{Z}}^{(t)}\right)
\\ & = (1 - \alpha \lambda - \alpha) {\boldsymbol{Z}}^{(t)} + \alpha \hat {\boldsymbol{A}} {\boldsymbol{Z}}^{(t)} + \alpha f\left({\boldsymbol{X}};{\boldsymbol{W}}\right), \label{eq:unfolding-twirls}
\end{align}
where $\alpha > 0$ is the step size of gradient descent. Notice that we fix ${\boldsymbol{W}}$ and view $E$ as a function of ${\boldsymbol{Z}}$. The model output is ${\boldsymbol{Z}}^{(T)}$ for a fixed iteration number $T$. When $\alpha$ is set to a suitable value and $T$ is large, it is clear that ${\boldsymbol{Z}}^{(T)}$ is an approximation of ${\boldsymbol{Z}}^*$ defined in \cref{eq:bilevel-inner}. 

In this paper, we always assume $T\to \infty$ and thus the output of TWIRLS is exactly the unique global optimum  of \cref{eq:energy}. Therefore, we directly consider the explicit optimal solution of \cref{eq:energy}: 
\begin{align}
{\boldsymbol{Z}}^*({\boldsymbol{W}}) = \arg\min_{{\boldsymbol{Z}}} E({\boldsymbol{Z}}) = \left({\boldsymbol{I}} + \lambda \hat {\boldsymbol{L}}\right)^{-1} f({\boldsymbol{X}}; {\boldsymbol{W}}).
\end{align}

In \cite{twirls}, the implementation of $f$ is arbitrary as long as it can be trained (i.e. computable, smooth, etc.). In this paper we consider a simple implementation of $f$ which is a linear function, namely $f({\boldsymbol{X}};{\boldsymbol{W}}) = {\boldsymbol{X}}{\boldsymbol{W}}$ where ${\boldsymbol{W}} \in \mathbb R^{d \times c}$. Notice that this is also the practical setting used in \citep{twirls} for many datasets. For loss function we choose the mean squared error (MSE) loss defined as $\ell({\boldsymbol{Z}}, {\boldsymbol{Y}}) := \frac{1}{2}\left\|{\boldsymbol{Z}} - {\boldsymbol{Y}}\right\|_\mathcal F^2$. 
\begin{algorithm}[tbp]
\caption{HERTA: A High-Efficiency and Rigorous Training Algorithm for TWIRLS}\label{alg:herta}
\begin{algorithmic}
    \STATE {\bfseries Input:}  $\hat {\boldsymbol{L}}, {\boldsymbol{X}}, {\boldsymbol{Y}}$, $\lambda$, $c$, $K>0$, step size $\eta$ and number of iteration $T$.

    \STATE Set $\beta = \frac{1}{64}$, $s = \frac{Kd}{\beta^2} \log n$, $\mu = \min\left\{ \frac{\epsilon^{1/2}}{50 \kappa({\boldsymbol{X}})\lambda^2} , 1\right\}$ and ${\boldsymbol{R}} \in \mathbb R^{n \times n}$ a diagonal Rademacher matrix;
    
    \STATE $\tilde {\boldsymbol{L}} \gets \mathsf{Sparsify}_{\frac{\beta}{3\lambda}}(\hat {\boldsymbol{L}})$ by calling \Cref{alg:sparse};
    \STATE ${\boldsymbol{Q}} \gets \mathsf{SDDSolver}_{\frac{\beta}{\sqrt{3\lambda}}}\left({\boldsymbol{I}} + \lambda \tilde {\boldsymbol{L}}; {\boldsymbol{X}}\right)$;
    \STATE ${\boldsymbol{Q}}' \gets \mathsf{Hadamard}\left({\boldsymbol{R}} {\boldsymbol{Q}}\right)$;   
    \STATE Subsample $s$ rows of ${\boldsymbol{Q}}'$ uniformly and obtain $\tilde {\boldsymbol{Q}}$;
    \STATE ${\boldsymbol{P}} \gets \tilde {\boldsymbol{Q}}^\top \tilde {\boldsymbol{Q}}$;
    \STATE ${\boldsymbol{P}}' \gets {\boldsymbol{P}}^{-1/2}$;

    \FOR{$i = 1$ {\bfseries to} $c$}
    \STATE $\hat {\boldsymbol{y}}_i \gets ({\boldsymbol{I}} + \lambda \hat {\boldsymbol{L}}) {\boldsymbol{y}}_i$, where ${\boldsymbol{y}}_i$ is the $i$-th column of ${\boldsymbol{Y}}$;
    \STATE{Initialize ${\boldsymbol{w}}^{(0)}_i$ by all zeros;}
   \FOR{$t=1$ {\bfseries to} $T$}
        \STATE ${\boldsymbol{u}}^{(t)}_i \gets {\boldsymbol{X}} {\boldsymbol{P}}^{'} {\boldsymbol{w}}_i^{(t-1)} - \hat {\boldsymbol{y}}_i$;
        \STATE ${\boldsymbol{u}}_i^{(t)'} \gets \mathsf{SDDSolver}_{\mu}\left({\boldsymbol{I}} + \lambda \hat {\boldsymbol{L}}; {\boldsymbol{u}}_i^{(t)}\right)$;
        \STATE ${\boldsymbol{u}}_i^{(t)''} \gets \mathsf{SDDSolver}_{\mu}\left({\boldsymbol{I}} + \lambda \hat {\boldsymbol{L}}; {\boldsymbol{u}}_i^{(t)'}\right)$;
        \STATE ${\boldsymbol{g}}_i^{(t)} \gets {\boldsymbol{P}}^{'} {\boldsymbol{X}}^\top {\boldsymbol{u}}_i^{(t)''}$;
        \STATE ${\boldsymbol{w}}_i^{(t)} \gets {\boldsymbol{w}}_i^{(t-1)} - \eta {\boldsymbol{g}}_i^{(t)}$;
    \ENDFOR
    \ENDFOR

    \STATE {\bfseries Output:} $\left\{ {\boldsymbol{w}}^{(T)}_i\right\}_{i=1}^c$.
\end{algorithmic}
\end{algorithm}

By combining the above formulations and assumptions, the overall target function we consider in this paper has the following form: \begin{align}
\ell\left[ {\boldsymbol{Z}}^*({\boldsymbol{W}}) ; {\boldsymbol{Y}}\right] & = \frac{1}{2}\left\|\left({\boldsymbol{I}} + \lambda \hat {\boldsymbol{L}}\right)^{-1}{\boldsymbol{X}} {\boldsymbol{W}} - {\boldsymbol{Y}}\right\|_\mathcal F^2.\label{eq:final-loss}
\end{align}
This loss can be viewed as a summation of $c$ independent sub-loss functions. Define $\ell_i({\boldsymbol{w}}_i) := \frac{1}{2}\|({\boldsymbol{I}} + \lambda \hat {\boldsymbol{L}})^{-1} {\boldsymbol{X}}{\boldsymbol{w}}_i - {\boldsymbol{y}}_i\|^2$ where ${\boldsymbol{w}}_i$ and ${\boldsymbol{y}}_i$ are the $i$-th column of ${\boldsymbol{W}}$ and ${\boldsymbol{Y}}$ respectively, then we have \begin{align}\ell\left[{\boldsymbol{Z}}^*({\boldsymbol{W}}); {\boldsymbol{Y}}\right] = \sum_{i=1}^c \ell_i({\boldsymbol{w}}_i).\end{align} 
Note that we view $\ell_i$ as a function of ${\boldsymbol{w}}_i$. Moreover, by denoting $\hat {\boldsymbol{y}}_i = ({\boldsymbol{I}} + \lambda \hat {\boldsymbol{L}}) {\boldsymbol{y}}_i$ we can rewrite $\ell_i$ as \begin{align}
\ell_i({\boldsymbol{w}}_i) = \frac{1}{2} \left\| \left({\boldsymbol{I}} + \lambda \hat {\boldsymbol{L}}\right)^{-1}\left({\boldsymbol{X}}{\boldsymbol{w}}_i - \hat {\boldsymbol{y}}_i\right) \right\|^2.\label{eq:sub-loss}
\end{align}

\begin{algorithm}[tbp]
\caption{$\mathsf{Sparsify}_{\epsilon}(\hat{\boldsymbol{L}})$: Regularized Spectral Sparsifier}\label{alg:sparse}
\begin{algorithmic}
    \STATE {\bfseries Input:}  $\hat {\boldsymbol{L}}, \hat {\boldsymbol{B}}, \lambda$, $C > 0$ and expected error rate $\epsilon$.
    \STATE Set $k = C\log m$ and $s = C n_{\lambda}\log n/\epsilon^2$;
    \STATE Construct Gaussian sketch $\boldsymbol{\Pi}_1 \in\R^{k\times m}, \boldsymbol{\Pi}_2 \in\R^{k\times n}$;
    \STATE  $\boldsymbol{B}_{\mathcal{S}} \gets \mathsf{SDDSolver}_{2^{1/4}}(\hat{\boldsymbol{L}} + \lambda^{-1}{\boldsymbol{I}}; (\boldsymbol{\Pi}_1\hat{\boldsymbol{B}})^\top)$;
    \STATE  $\boldsymbol{\Pi}_{\mathcal{S}} \gets  \mathsf{SDDSolver}_{2^{1/4}}(\hat{\boldsymbol{L}} + \lambda^{-1}{\boldsymbol{I}}; \boldsymbol{\Pi}_2^\top)$;
   \FOR{$i=1$ {\bfseries to} $m$}
        \STATE  $\tilde{l}_i \gets  \|\boldsymbol{B}_{\mathcal{S}}\hat{\boldsymbol{b}}_i\|^2 + \lambda^{-1}\|\boldsymbol{\Pi}_{\mathcal{S}}\hat{\boldsymbol{b}}_i\|^2$, where $\hat{\boldsymbol{b}}_i^\top$ is the $i$-th row of $\hat{\boldsymbol{B}}$;
    \ENDFOR
    \STATE  $Z \gets  \sum_{i=1}^m \tilde{l}_i$;
    \STATE Subsample $s$ rows of $\hat{\boldsymbol{B}}$ with probabilities $\left\{\tilde{l}_i/Z \right\}_{i=1}^m$ and obtain $\tilde{\boldsymbol{B}}$;
    \STATE $\tilde{\boldsymbol{L}} \gets  \tilde{\boldsymbol{B}}^\top\tilde{\boldsymbol{B}}$;
    \STATE {\bfseries Output:} $\tilde{\boldsymbol{L}}$.
\end{algorithmic}
\end{algorithm}

\section{Algorithm and Analysis}\label{sec:efficient-twirls-training}

In this section we present our main algorithm HERTA, give the convergence result and analyze its time complexity. We also introduce our key techniques and give a proof sketch for the main result. To start we introduce the following notion.

\begin{definition}[Effective Laplacian dimension]
For normalized Laplacian $\hat{\boldsymbol{L}}$ and regularization term $\lambda>0$, define \begin{align}
   n_\lambda := \mathrm{Tr} \left[ \hat {\boldsymbol{L}} \left(\hat {\boldsymbol{L}} + \lambda^{-1} {\boldsymbol{I}}\right)^{-1}\right]
    \end{align}
    as the effective Laplacian dimension of the graph.
\end{definition}
Denote $\{\lambda_i\}_{i=1}^n$ as the eigenvalues of $\hat{\boldsymbol{L}}$. Then the effective Laplacian dimension can be written as $n_\lambda = \sum_i \frac{\lambda_i}{\lambda_i+\lambda^{-1}}$. We can see that $n_\lambda$ is roughly the number of eigenvalues of $\hat {\boldsymbol{L}}$ which are of order $\Omega(\lambda^{-1})$, i.e., the number of ``large eigenvalues'' mentioned in \Cref{thm:main-informal}. It is not hard to see that $n_\lambda \leq n$ and $n_\lambda \to n$ as $\lambda \to \infty$. 

Intuitively, $n_\lambda$ represents the number of eigen-directions for which the Laplacian regularizer $\frac{\lambda }{2} \mathrm{Tr} ({\boldsymbol{Z}}^\top \hat {\boldsymbol{L}} {\boldsymbol{Z}})$ is large, thus having a significant impact on  the energy function \eqref{eq:energy}. For the remaining eigen-directions which are less significant, the effect of the regularizer is minimal. This can also be seen from the loss function \eqref{eq:sub-loss}: when we apply matrix $({\boldsymbol{I}} + \lambda \hat {\boldsymbol{L}})^{-1}$ to a vector, the effect of the Laplacian is dominated by the ${\boldsymbol{I}}$ term if this vector is aligned with a ``small'' Laplacian eigen-direction (i.e., with an eigenvalue significantly smaller than $\lambda^{-1}$), and thus the objective is close to the usual unregularized least squares loss.

Based on the above discussion, it stands to reason that the larger the effective Laplacian dimension $n_\lambda$ of the graph, the more work we must do during the training to preserve the graph structure of the GNN. With the above discussion, we formally give our main result for HERTA.

\begin{theorem}[Main result]\label{thm:main} 
For any $\epsilon > 0$, with a proper step size $\eta$, number of iterations $T$ and constant $K > 0$, \Cref{alg:herta} finds a solution $\hat {\boldsymbol{W}}\in\R^{d\times c}$ in time
\begin{align}
\tilde O\left(\left(m + nd\right)c \left(\log\frac{1}{\epsilon}\right)^2 +n_{\lambda}\lambda^2d+d^3\right)
\end{align}
such that with probability $1-\frac{1}{n}$,
\begin{align}
 \ell\left[ {\boldsymbol{Z}}^*({\boldsymbol{\hat W}}) ; {\boldsymbol{Y}}\right] \leq (1+\epsilon) \cdot \ell^*, 
 \end{align}
 where $\ell^* := \min_{{\boldsymbol{W}}} \ell\left[{\boldsymbol{Z}}^*\left( {\boldsymbol{W}}; {\boldsymbol{Y}}\right)\right]$.
\end{theorem}

Notice that if we assume $n_\lambda = O(n/\lambda^2)$, then $\tilde{O}(n_\lambda \lambda^2 d) = \tilde{O}(nd)$. By further assuming the number of classes $c=O(1)$, we recover the bound given by \Cref{thm:main-informal}.

Before heading into the details of HERTA (\Cref{alg:herta}), we first analyze the original (standard) implementation of training TWIRLS and point out the efficiency bottlenecks. We then introduce the key techniques and methods we need to prove our main result. Finally we present our main algorithm HERTA, and give a proof sketch for \Cref{thm:main}.

\subsection{Analysis of TWIRLS Training}\label{sec:naive}

In this subsection, we introduce the framework that we use to analyze the complexity of TWIRLS training. We also provide a complexity analysis of the implementation used in \cite{twirls}. 

As noted in \Cref{sec:problem-setting}, the whole optimization problem can be viewed as a bi-level optimization problem, where the inner-problem \cref{eq:bilevel-inner} approximates the linear system solver $\left({\boldsymbol{I}} + \lambda \hat {\boldsymbol{L}}\right)^{-1}({\boldsymbol{X}} {\boldsymbol{w}}_i - \hat {\boldsymbol{y}}_i)$, and the outer problem \cref{eq:bilevel-outer} uses the linear system solution to approximate the optimal solution of the training loss \cref{eq:final-loss}. When we discuss the solution of outer problems, we treat the solver of the inner problem (linear solver) as a black box. We formally define a linear solver as follows.
\begin{definition}[Linear solver]
    For a positive definite matrix ${\boldsymbol{H}} \in \mathbb R^{n \times n}$ and a real number $\epsilon > 0$, we call $f: \mathbb R^n \to \mathbb R^n$ a linear solver for ${\boldsymbol{H}}$ with $\epsilon$ error rate if it satisfies \begin{align}
        \forall {\boldsymbol{u}} \in \mathbb R^n,~ \left\| f({\boldsymbol{u}}) - {\boldsymbol{H}}^{-1}{\boldsymbol{u}}\right\|_{{\boldsymbol{H}}} \leq \epsilon \left\| {\boldsymbol{H}}^{-1} {\boldsymbol{u}} \right\|_{{\boldsymbol{H}}}.
    \end{align}
\end{definition}

In \citep{twirls}, the outer problem is also solved by standard gradient descent. The gradient of $\ell_i\left({\boldsymbol{w}}_i\right)$ is \begin{align}\label{eq:true-grad}
\nabla \ell_i({\boldsymbol{w}}_i) = {\boldsymbol{X}}^\top \left({\boldsymbol{I}} + \lambda \hat {\boldsymbol{L}}\right)^{-2} \left( {\boldsymbol{X}} {\boldsymbol{w}}_i - \hat {\boldsymbol{y}}_i\right).
\end{align}
In actual implementation, the matrix inverse is calculated by linear solvers. Suppose that we have a linear solver for $({\boldsymbol{I}} + \lambda \hat {\boldsymbol{L}})$ denoted by $\mathcal S$, then by embedding it into \cref{eq:true-grad} we can obtain the gradient approximation as 
\begin{align}
\widetilde {\nabla \ell_i({\boldsymbol{w}}_i)} :=  {\boldsymbol{X}}^\top \mathcal S\left[ \mathcal S\left( {\boldsymbol{X}}{\boldsymbol{w}}_i - \hat {\boldsymbol{y}}_i\right)\right]. \label{eq:gradient-approximator-naive}
\end{align}

The following convergence result can be derived for this approximate gradient descent. See Appendix for the proof.

\begin{lemma}[Convergence]\label{lem:outer-analysis-naive}
    If we minimize \cref{eq:sub-loss} using gradient descent with the gradient approximation defined in \cref{eq:gradient-approximator-naive}, where $\mathcal S$ is a linear solver of ${\boldsymbol{I}} + \lambda \hat {\boldsymbol{L}}$ with $\mu$ error rate satisfying $\mu \leq \min \left\{ \frac{\epsilon^{1/2}}{50 \kappa({\boldsymbol{X}})\lambda^2} , 1\right\}$, then, to obtain an $\epsilon \in (0,1)$ error rate the outer problem \cref{eq:sub-loss}, the number of iterations needed is \begin{align}
        T = O\left(\kappa \log\frac{1}{\epsilon}\right),
    \end{align}
    where $\kappa := \kappa({\boldsymbol{X}}^\top (\boldsymbol{I} + \lambda \hat{\boldsymbol{L}})^{-2} {\boldsymbol{X}})$.
\end{lemma}

Based on \Cref{lem:outer-analysis-naive}, in \Cref{sec:complexity-of-naive} we provide a detailed analysis of the time complexity of the implementation used by \cite{twirls}. From this analysis, one can figure out two bottlenecks of the computational complexity for solving TWIRLS: 1) The number of iterations needed for the inner loop is dependent on $\lambda$; 2) The number of iterations needed for the outer loop is dependent on the condition number of the outer problem. Our acceleration algorithm is based on the following two observations, which correspond to the aforementioned bottlenecks: \begin{enumerate}
    \item The data matrix in the inner loop is a graph Laplacian plus identity. It is possible to solve the linear system defined by this matrix faster by exploiting its structure.
    \item A good enough preconditioner can be constructed for the outer iteration using techniques from RandNLA.
\end{enumerate}

\subsection{Key Techniques}

In this subsection, we introduce the mathematical tools that we will use in the proof of our main result.

\paragraph{SDD Solvers.} 

If a symmetric matrix  ${\boldsymbol{H}} \in \mathbb R^{n \times n}$ satisfies \begin{align}
\forall 1\leq  k \leq n,~~ {\boldsymbol{H}}_{k,k} \geq \sum_{i=1}^n \left| {\boldsymbol{H}}_{k,i}\right|,
\end{align}
it is called a symmetric diagonal dominated (SDD) matrix. It has been shown that the linear system defined by sparse SDD matrices can be solved in an efficient way, via a SDD solver \cite{sdd-solver-comb, sdd-solver}.

Since ${\boldsymbol{I}} + \lambda \hat {\boldsymbol{L}}$ is indeed a SDD matrix in our problem, we can apply off-the-shelf SDD solvers. In the following, we use $\mathsf{SDDSolver}_{\epsilon}({\boldsymbol{H}}, {\boldsymbol{u}})$ to denote the SDD solver that calculates ${\boldsymbol{H}}^{-1} {\boldsymbol{u}}$ with error rate less or equal to $\epsilon$\footnote{Principally, $\mathsf{SDDSolver}({\boldsymbol{H}};\cdot)$ is a vector function, but for convenience we sometimes also apply it to matrices, in which case we mean applying it column-wisely. We also use the same convention for other vector functions.}. From \Cref{lem:sdd-solver}, the time complexity of calculating $\mathsf{SDDSolver}_{\epsilon}({\boldsymbol{H}}; {\boldsymbol{u}})$ is $\tilde O\left[\mathrm{nnz}({\boldsymbol{H}})\right]$. See \Cref{sec:sdd-solver} for details.

\paragraph{Fast Matrix Multiplication.}

For a matrix ${\boldsymbol{Q}} \in \mathbb R^{n \times d}$ where $n > d$, it is known that calculating ${\boldsymbol{Q}}^\top {\boldsymbol{Q}}$  takes $O(nd^2)$ time if implemented by brute force. In HERTA, we manage to achieve a faster calculation of ${\boldsymbol{Q}}^\top {\boldsymbol{Q}}$ through Subsampled Randomized Hadamard Transformation (SRHT) \cite{srht_tropp, wood14}. The key idea of SRHT is to first apply a randomized Hadamard transformation to the matrix to make the ``information'' evenly distributed in each row, then uniformly sample the rows of ${\boldsymbol{Q}}$ to reduce the dimension. In the following, we use $\mathsf{Hadamard}$ to represent the Hadamard transformation, see \Cref{sec:fast-matrix-multiplication} for the definition of $\mathsf{Hadamard}$ as well as a detailed introduction. 

\subsection{Regularized Spectral Sparsifier}

It has been shown that for any connected graph with $n$ nodes, (no matter how dense it is) there is always a sparsified graph with $\tilde O\left(\frac{n}{\epsilon^2}\right)$ edges whose Laplacian is an $\epsilon$-spectral approximation of the original graph Laplacian \cite{first-spectral-sparsifier}. Although the existence is guaranteed, it is generally hard to construct such a sparsified graph. An algorithm that finds such a sparsified graph 
is called a spectral sparsifier.

While there has been many existing explorations on constructing spectral sparsifier of a given graph with high probability and tolerable time complexity \cite{spectral-sparifier-effective-resistance, first-spectral-sparsifier,spectral-sparsifier}, the setting considered in this paper is  different from the standard one: instead of graph Laplacian $\boldsymbol{L}$, the matrix we concern is the normalized and regularized Laplacian $\hat{\boldsymbol{L}} + \lambda^{-1} {\boldsymbol{I}}$.

At first glance, a spectral sparsifier also works for our problem since  $\tilde {{\boldsymbol{L}}} \approx_{\epsilon} \hat {\boldsymbol{L}}$ implies $\tilde {{\boldsymbol{L}}} + \lambda^{-1} {\boldsymbol{I}} \approx_{\epsilon} \hat {\boldsymbol{L}} + \lambda^{-1} {\boldsymbol{I}}$. However, it turns out that it is possible to make the bound even tighter if we take into account  the regularization term $\lambda^{-1} {\boldsymbol{I}}$ here. A similar argument is also made in \cite{ridge-laplacian-sparifier}. We note that the setting explored in \cite{ridge-laplacian-sparifier} is also different from ours: we consider the normalized Laplacian of a graph, which is not a Laplacian of any graph (even allowing weighted edges).

To this end, we propose a new spectral sparsifier that works for $\hat {\boldsymbol{L}} +\lambda^{-1}  {\boldsymbol{I}}$ which is used in our problem. This spectral sparsifier reduces the number of edges to $O\left(\frac{n_\lambda}{\epsilon^2} \log n\right)$, which, as we discussed before, is smaller than $O\left(\frac{n}{\epsilon^2} \log n\right)$ obtained by standard spectral sparsifier.

\begin{lemma}[Regularized spectral sparsifier]\label{lem:regularized-spectral-sprsification}
    Let $\hat {\boldsymbol{L}}$ be a (normalized) graph Laplacian with $m$ edges and $n$ nodes. There exists a constant $C$ such that for any $\epsilon > 0$ and $\lambda > 0$, \Cref{alg:sparse} outputs 
    a (normalized) graph Laplacian $\tilde {\boldsymbol{L}}$ with $n$ nodes and $O\left(\frac{n_\lambda}{\epsilon^2} \log n\right)$ edges in time $\tilde{O}(m)$, such that
    \begin{align}
    \tilde {\boldsymbol{L}} + \lambda^{-1} {\boldsymbol{I}} \approx_{\epsilon} \hat {\boldsymbol{L}} + \lambda^{-1} {\boldsymbol{I}}
    \end{align}
    holds with probability at least $1 - \frac{1}{2n}$.
\end{lemma}

The basic idea behind \Cref{lem:regularized-spectral-sprsification} is to use ridge leverage score sampling for the normalized incidence matrix ${\boldsymbol{B}}{\boldsymbol{D}}^{-1/2}$ \cite{ridge-LSS}, which reduces the problem to estimating a regularized version of the effective resistance. We modify the existing methods of estimating effective resistance \cite{laplacian-solver} to make it work for the normalized and regularized Laplacian.

\subsection{Main Algorithm}\label{sec:main-algorithm}

Now we are ready to present our main algorithm HERTA, see \Cref{alg:herta}. As mentioned before, HERTA is composed of two major components: constructing a preconditioner matrix ${\boldsymbol{P}}$ and applying it to the optimization problem. Below we introduce each part.

\paragraph{Constructing the Preconditioner.}

For the outer problem which can be ill-conditioned, we first precondition it using a preconditioner ${\boldsymbol{P}} \in \mathbb R^{d \times d}$ which is a constant level spectral approximation of the Hessian, i.e. ${\boldsymbol{X}}^\top \left({\boldsymbol{I}} + \lambda \hat {\boldsymbol{L}}\right)^{-2} {\boldsymbol{X}}$. We claim that the matrix ${\boldsymbol{P}}$ constructed inside \Cref{alg:herta} indeed satisfies this property.

\begin{lemma}[Preconditioner]\label{lem:preconditioner}
    Let ${\boldsymbol{P}}$ be the matrix constructed in \Cref{alg:herta}. With probability at least $1-\frac{1}{n}$, \begin{align}
        {\boldsymbol{P}} \approx_{\frac{1}{2}} {\boldsymbol{X}}^\top \left({\boldsymbol{I}} + \lambda \hat {\boldsymbol{L}}\right) ^{-2} {\boldsymbol{X}}.
    \end{align}
\end{lemma}

\begin{figure*}[htbp]
\hspace{-6mm}\begin{minipage}{0.35\linewidth}
    \centering
    \includegraphics[width=\linewidth]{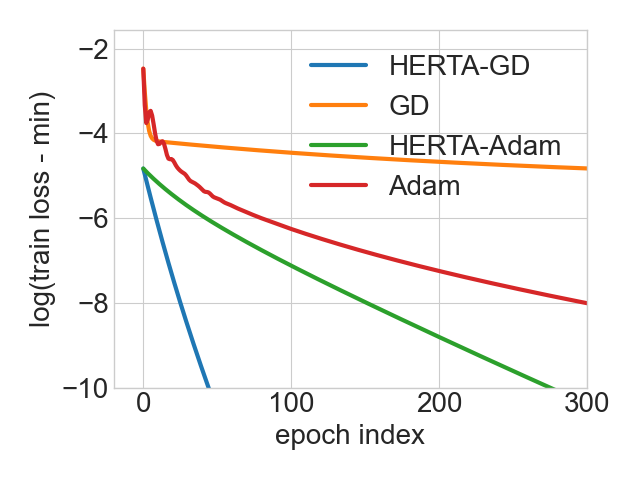}
\end{minipage}
\begin{minipage}{0.35\linewidth}
    \centering
    \includegraphics[width=\linewidth]{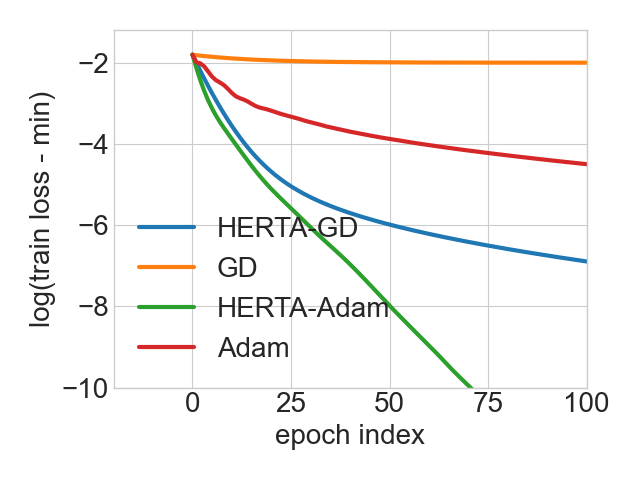}
\end{minipage}
\begin{minipage}{0.35\linewidth}
    \centering
    \includegraphics[width=\linewidth]{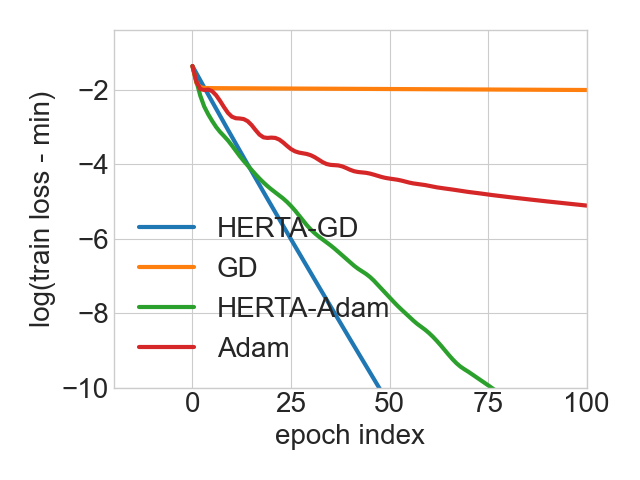}
\end{minipage}
    \caption{The training loss comparison between HERTA and standard optimizers on MSE loss with $\lambda = 1$. Dataset used from left to right: ogbn-arxiv, citeseer, pubmed.}
    \label{fig:ours-mse-1}
\end{figure*}

\paragraph{Solving the Outer Problem.}

After obtaining the preconditioner ${\boldsymbol{P}}$ which approximates the Hessian by a constant level, we use it to precondition the outer problem and provably reduce the condition number to a constant. With the new problem being well-conditioned, iterative methods (like gradient descent) take much less steps to converge.

\begin{lemma}[Well-conditioned Hessian]\label{thm:after-preconditioning}
    Let $\ell'$ be such that $\ell'({\boldsymbol{w}}_i) = \ell({\boldsymbol{P}}^{-1/2} {\boldsymbol{w}}_i)$. Suppose for some constant $c_0 \in (0,1)$ we have ${\boldsymbol{P}} \approx_{c_0} {\boldsymbol{X}}^\top \left({\boldsymbol{I}} + \lambda \hat {\boldsymbol{L}}\right)^{-2} {\boldsymbol{X}}$, then we can bound the condition number of Hessian of $\ell'$ as
    \begin{align}\kappa\left(\nabla^2 \ell'\right) \leq (1 + c_0)^2.\end{align}
    Moreover, if ${\boldsymbol{w}}'$ is a solution of $\ell'$ with $\epsilon$ error rate, then ${\boldsymbol{P}}^{-1/2} {\boldsymbol{w}}'$ is a solution of $\ell$ with $\epsilon$ error rate.
\end{lemma}

Notice that, the problem $\ell'$ defined in \Cref{thm:after-preconditioning} can be viewed as the original problem $\ell$ with ${\boldsymbol{X}}$ being replaced by $ {\boldsymbol{X}}{\boldsymbol{P}}^{-1/2}$, therefore we can use \Cref{lem:outer-analysis-naive} and obtain the convergence rate of HERTA. With the convergence result and analysis of the running time for each step, we are able to prove \Cref{thm:main}. See \Cref{sec:proof-main} for the full proof.

\textbf{Remark on the Unavoidable $\lambda^2$ in Runtime.}~ Notice that $\lambda^2$ appears in the time bound given in \Cref{thm:main}. From the proof in \Cref{sec:proof-lem-precond} we can see that this term originates from the squared inverse of ${\boldsymbol{I}} + \lambda \hat {\boldsymbol{L}}$. In \Cref{sec:sqaure-tight}, we show that in the worst case even if we approximate a matrix to a very high precision, the approximation rate of the corresponding squared version can still be worse by a factor of the condition number of the matrix, which in our case leads to the unavoidable $\lambda^2$ in the runtime. 

\textbf{Remark on Graph Sparsifying in Each Iteration.}~ Note that in \Cref{alg:herta}, we use the complete Laplacian $\hat {\boldsymbol{L}}$ for gradient iterations (i.e., the for-loop), and the graph sampling only occurs when constructing the preconditioner ${\boldsymbol{P}}$. We note that sampling $\hat {\boldsymbol{L}}$ in the for-loop can lead to extra loss in the gradient estimation, which forces us to sparsify the graph to a very high precision to ensure an accurate enough gradient for convergence. This could result in a suboptimal running time. Moreover, as discussed in \Cref{sec:introduction}, the current running time bound for HERTA is optimal up to logarithmic factors, which means that there is little to gain from performing extra sampling in for-loops. See \Cref{sec:applying-graph-sparsification-in-each-iteration} for a more detailed and quantitative discussion.

\vspace{-1em}
\section{Experiments}\label{sec:experiments}
\vspace{-0.5em}

In this section, we verify our theoretical results through experiments on real world datasets. Since for the inner problem we use off-the-shelf SDD solvers, in this section we focus on the outer problem, i.e. the training loss convergence rate. In each setting we compare the training loss convergence rate of TWIRLS trained by our method against that trained by standard gradient descent using exactly the same training hyper-parameters.

\textbf{Datasets.} We conduct experiments on all of the datasets used in the Section 5.1 of \cite{twirls}: Cora, Citeseer and Pubmed collected by \cite{coraetc}, as well as the ogbn-arxiv dataset from the Open Large Graph Benchmark (OGB, \citealp{ogb}). 

Notice that for Cora, Citeseer and Pubmed, common practice uses the semi-supervised setting where there is a  small training set and relatively large validation set. As we are comparing training convergence rate, we find it more comparative to use a larger training set (which makes solving the optimization problem of training more difficult). Therefore, we randomly select 80\% of nodes as training set for Cora, Citeseer and Pubmed. For OGB, we use the standard training split.

Due to limited space, we only include the results with  $\lambda = 1$ and with datasets Citeseer, Pubmed and ogbn-arxiv in this section, and delay additional results to \Cref{sec:appendix-experiment}.

\subsection{Convergence Rate Comparison Under MSE Loss}

In this section, we compare the convergence rate for models trained with MSE loss, which is well aligned with the setting used to derive our theory. We also adopt a variation of HERTA that allows using it on other optimizers and apply it on Adam optimizer. See \Cref{fig:ours-mse-1} for the results. Notice that for each figure, we shift the curve by the minimum value of the loss\footnote{The minimum value of the loss is approximated by running HERTA for more epochs.} and take logarithm scale to make the comparison clearer.

From the results, it is clear that on all datasets and all optimizers we consider, HERTA converges much faster than standard methods.  It generally requires $\leq 10$ iterations to get very close to the smallest training loss. This shows that the guarantee obtained in our theoretical results (\Cref{thm:main}) not only matches the experiments, but also holds when the setting is slightly changed (using other optimizers). Thus, these experimental results verify the universality of HERTA as an optimization framework.

We also conduct extensive experiments with larger $\lambda$. See \Cref{sec:appendix-experiment} for the results. These results verify that even when $\lambda$ is relatively large, HERTA converges very fast, which suggests that the dominant term in \Cref{thm:main} is the $\tilde O(m)$ term instead of the $\tilde O(n_\lambda \lambda^2 d)$ term.

\subsection{Convergence Rate Comparison under Cross Entropy Loss}\label{sec:exp-ce-loss}

Note that in the theoretical analysis of \Cref{sec:efficient-twirls-training}, we focus on the MSE loss. However, for graph node classification tasks, MSE is not the most commonly used loss. Instead, most node classification models use the cross entropy (CE) loss. For example, the original TWIRLS paper \cite{twirls} uses CE loss. 

To this end, we also adopt HERTA on training problems with CE loss. See \Cref{sec:implementation-details} for the details of the implementation. The results are displayed in \Cref{fig:ce-fastce-1}. From these results, it is clear that HERTA also significantly speeds up training with CE loss. The results demonstrate that although originally developed on MSE loss, HERTA is not limited to it and have the flexibility to be applied to other type of loss functions.

\begin{figure*}[htbp]
\hspace{-6mm}\begin{minipage}{0.35\linewidth}
    \centering
    \includegraphics[width=\linewidth]{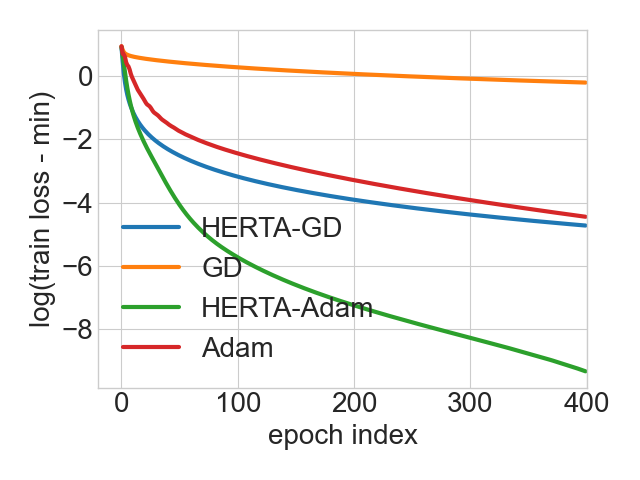}
\end{minipage}
\begin{minipage}{0.35\linewidth}
    \centering
    \includegraphics[width=\linewidth]{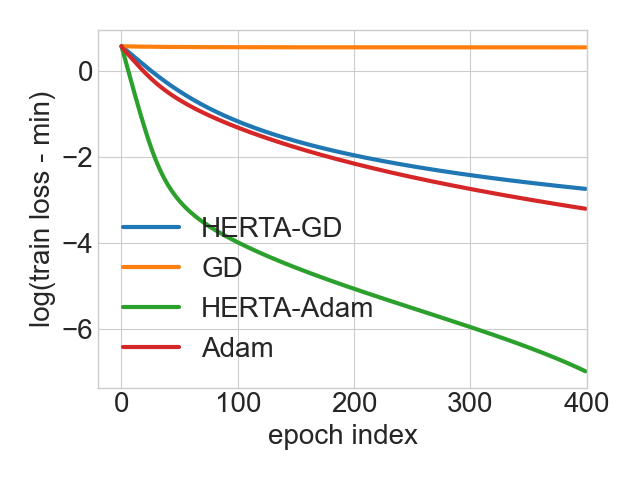}
\end{minipage}
\begin{minipage}{0.35\linewidth}
    \centering
    \includegraphics[width=\linewidth]{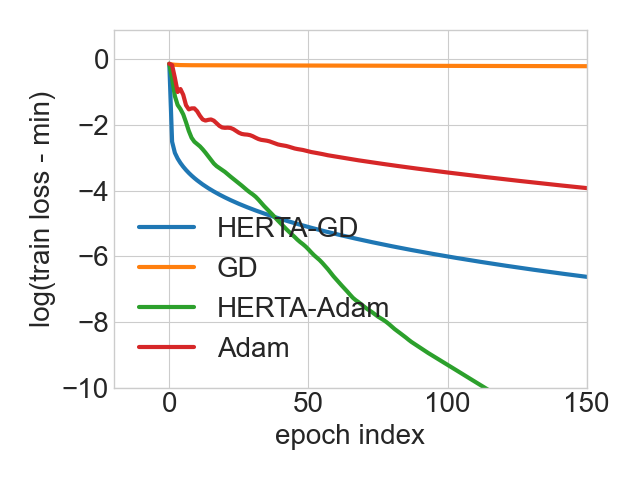}
\end{minipage}
    \caption{The training loss comparison between HERTA and standard optimizers on cross entropy loss with $\lambda = 1$. Dataset used from left to right: ogbn-arxiv, citeseer and pubmed. }
    \label{fig:ce-fastce-1}
\end{figure*}

\textbf{Remark on the Surprising Effectiveness of HERTA on Cross Entropy Loss.}~ 
In \Cref{sec:exp-ce-loss}, we observe that despite not being guaranteed by the theoretical result, HERTA shows a certain degree of universality in that it also works on CE loss. We claim that this phenomenon might originate from the fact the Hessians of TWIRLS under MSE loss and CE loss are very similar. We provide an analysis of the gradient and Hessian of TWIRLS under MSE and CE losses in \Cref{sec:anlysis-of-ce-loss}. The results show that the gradient under CE loss can be viewed as the gradient under MSE loss with one term being normalized by softmax, and the Hessian under CE loss can be viewed as a rescaled version of the Hessian under MSE loss. These comparisons serve as an explanation of why HERTA works so well with CE loss.  

\vspace{-1em}
\section{Conclusions}\label{sec:conclusion}
\vspace{-0.5em}

In this paper we present HERTA: a High-Efficiency and Rigorous Training Algorithm  that solves the problem of training Unfolded GNNs on a graph within a time essentially of the same magnitude as the time it takes to load the input data. As a component of HERTA, we also propose a new spectral sparsifier that works for normalized and regularized graph Laplacian matrices. 

Experimental results on real world datasets show the effectiveness of HERTA. Moreover, it is shown that  HERTA works for various loss functions and optimizers, despite being derived from a specific loss function and optimizer. This shows the universality of HERTA and verifies that it is ready to use in practice. 

\vspace{-0.7em}
\subsection{Future Directions}

In this subsection, we briefly discuss some potential future directions to explore with HERTA. 

\textbf{Extending to more Complex Models.} In this paper, we stick to a rather straightforward implementation of TWIRLS.
 Although it is verified in \cite{twirls} that this implementation is already strong enough in many cases, under more general and subtle scenarios we might want to use more complex implementations of TWIRLS or other Unfolded GNN models. Therefore, it is critical to extend HERTA to those more complex settings. We discuss possible directions for implementing HERTA for more complex $f$ as well as incorporating attention mechanisms in \Cref{sec:extending-to-more-complex}.

\textbf{More Detailed Analysis of CE Loss and other Losses.} Although in this paper we provided an analysis explaining why HERTA can work on CE loss, finding a theoretically rigorous bound for CE loss as well as other loss functions is another important direction for future work.

\bibliography{ref}
\bibliographystyle{icml2024}

\newpage
\appendix
\onecolumn

\section{Introduction to the Mathematical Tools}\label{sec:intro-to-math-tools}

In the main paper, due to space limitations, our introduction of certain mathematical tools and concepts remains brief. In this section, we provide a more comprehensive introduction to the mathematical tools utilized in the main paper.

\subsection{Subsampling}
Subsampling is used in multiple places in our Algorithm. When we say subsampling $s$ rows of a matrix ${\boldsymbol{M}} \in \mathbb R^{n \times d}$ with probability $\{p_k\}_{k=1}^n$ and obtain a new matrix $\tilde {\boldsymbol{M}} \in \mathbb R^{s \times d}$, it means each row of $\tilde {\boldsymbol{M}}$ is an i.i.d. random vector which follows the same distribution of a random vector ${\boldsymbol{\xi}}$ that satisfies the following property: \begin{align}
\mathbb P\left\{ {\boldsymbol{\xi}} = \sqrt{ \frac{1}{s p_k} }{\boldsymbol{m}}_k \right\} = p_k,
\end{align}
where ${\boldsymbol{m}}_k$ is the $k$-th row of matrix ${\boldsymbol{M}}$. It is clear that constructing such a subsampled matrix $\tilde {\boldsymbol{M}}$ takes time $O(sd)$. Moreover, if $p_k = \frac{1}{n}$, we say the subsampling is with uniform probability. 

The operator of subsampling $s$ rows of a $n$-row matrix is clearly a linear transformation. Therefore it can be represented by a random matrix ${\boldsymbol{S}} \in \mathbb R^{s \times n}$. Notice that when we apply ${\boldsymbol{S}}$ to a matrix (i.e. ${\boldsymbol{S}} \cdot {\boldsymbol{M}}$) we don't actually need to construct such a matrix ${\boldsymbol{S}}$ and calculate matrix product, since we only carry out sampling.

\subsection{SDD Solvers}\label{sec:sdd-solver}

As mentioned in the main paper, for a sparse SDD matrix ${\boldsymbol{M}}$, it is possible to fast approximate its linear solver. The following \Cref{lem:sdd-solver} rigorously states the bound we can obtain for SDD solvers. 

\begin{lemma}[\citealt{sdd-solver}]\label{lem:sdd-solver}
    For any SDD matrix ${\boldsymbol{M}} \in \mathbb R^{n \times n}$ and any real number $\epsilon > 0$, there exists an algorithm to construct an  operator $\mathsf{SDDSolver}_{\epsilon}\left({\boldsymbol{M}};\cdot\right)$, such that $\mathsf{SDDSolver}_{\epsilon}\left({\boldsymbol{M}};\cdot\right)$ is a linear solver for ${\boldsymbol{M}}$ with $\epsilon$ error rate and for any ${\boldsymbol{x}} \in \mathbb R^n$, calculating $\mathsf{SDDSolver}_{\epsilon}({\boldsymbol{M}};{\boldsymbol{x}})$ takes time \begin{align}
    O\left(m \left(\log \frac{1}{\epsilon}\right) \mathop{ \mathrm{ poly } }\log(m) \mathop{ \mathrm{ poly } }\log\left(\kappa({\boldsymbol{M}})\right)\right),
    \end{align}
    where $m = \mathrm{nnz}({\boldsymbol{M}})$.
\end{lemma}

\subsection{Fast Matrix Multiplication}\label{sec:fast-matrix-multiplication}

As mentioned in the main paper, we use SRHT to achieve fast matrix multiplication. Here we further explain this process. We refer interested readers to \cite{srht_tropp} for more details.
For a positive integer number $n$ that is a power of $2$, the Hadamard transformation of size $n$ is a linear transformation recursively defined as follows: \begin{align}
{\boldsymbol{H}}_n := \begin{bmatrix} 
{\boldsymbol{H}}_{n/2} & {\boldsymbol{H}}_{n/2} \\ {\boldsymbol{H}}_{n/2} & -{\boldsymbol{H}}_{n/2} \end{bmatrix}
\end{align}
and ${\boldsymbol{H}}_1 = 1$. For a vector ${\boldsymbol{x}} \in \mathbb R^n$, we define \begin{align}
\mathsf{Hadamard}({\boldsymbol{x}}) := \frac{1}{\sqrt n}{\boldsymbol{H}}_n {\boldsymbol{x}}.
\end{align}
Notice that when $n$ is not a power of $2$, we pad the vector ${\boldsymbol{x}}$ with $0$-s beforehand, therefore the requirement of $n$ being a power of $2$ is ignored in the actual usage. From its recursive nature, it is not hard to see that applying a Hadamard transformation to each column of an $n \times d$ matrix only takes $O(d \log n)$ time.

Let ${\boldsymbol{R}} \in \mathbb R^{n \times n}$ be a diagonal matrix whose diagonal entries are i.i.d. Rademacher variables, and ${\boldsymbol{S}} \in \mathbb R^{s \times n}$ be a subsampling matrix with uniform probability and $s = O\left( \frac{d}{\beta^2} \log n\right)$. For a matrix ${\boldsymbol{Q}} \in \mathbb R^{n \times d}$, its SRHT with $\beta$ error rate is defined as \begin{align}
\mathsf{SRHT}_\beta({\boldsymbol{Q}}) := {\boldsymbol{S}} \mathsf{Hadamard}\left({\boldsymbol{R}} {\boldsymbol{Q}}\right),
\end{align}
which is exactly the matrix $\tilde {\boldsymbol{Q}}$ we use in \Cref{alg:herta}. A fast matrix multiplication result can be achieved using SRHT.

\begin{lemma}[\citealt{srht_tropp}]\label{lem:SRHT}
    For matrix ${\boldsymbol{Q}} \in \mathbb R^{n \times d}$ where $n \geq d$ and $\mathrm{rank}(\boldsymbol{Q}) = d$, let $\tilde {\boldsymbol{Q}} = \mathsf{SRHT}_\beta({\boldsymbol{Q}})$ where $\beta \in (0,1/4)$, then we have \begin{align}
    \tilde {\boldsymbol{Q}}^\top \tilde {\boldsymbol{Q}} \approx_{\beta} {\boldsymbol{Q}}^\top {\boldsymbol{Q}}
    \end{align}
    with probability at least $1-\frac{1}{2n}$.
\end{lemma}

\section{Running Time of the Original Implementation Used in \cite{twirls}}\label{sec:complexity-of-naive}

In this section we analyze the time complexity of the implementation used in \cite{twirls}, which uses gradient descent to solve both inner and outer problem.

\paragraph{Inner Problem Analysis.}

We first analyze the time complexity of the inner problem solver used in \cite{twirls}. Recall that for the inner problem \cref{eq:bilevel-inner}, we need to find an approximation of the linear solver of ${\boldsymbol{I}} + \lambda \hat {\boldsymbol{L}}$. In \cite{twirls}, it is implemented by a standard gradient descent. Here we consider approximately solving the following least square problem by gradient descent:
 \begin{align}
{\boldsymbol{v}} = \arg\min_{{\boldsymbol{v}} \in \mathbb R^n}\frac{1}{2}\left\| \left({\boldsymbol{I}} + \lambda \hat {\boldsymbol{L}}\right){\boldsymbol{v}} - {\boldsymbol{u}}\right\|_2^2. \label{eq:the-least-square}
\end{align}
The gradient step with step size $\mu$ is:
\begin{align}
{\boldsymbol{v}}^{(t+1)} = (1-\mu) \left({\boldsymbol{I}} + \lambda \hat {\boldsymbol{L}}\right)^2 {\boldsymbol{v}}^{(t)} + \eta \left({\boldsymbol{I}} + \lambda \hat {\boldsymbol{L}}\right) {\boldsymbol{u}}. \label{eq:naive-inner-iter}
\end{align}
\begin{theorem}[Inner Analysis]\label{thm:inner-analysis-naive}
    If we update ${\boldsymbol{v}}^{(t)}$ through \cref{eq:naive-inner-iter} with initialization ${\boldsymbol{v}}^{(0)} = {\boldsymbol{0}}$ and a proper $\eta$, then after $T = O\left(\lambda^2 \log \frac{{\lambda}}{\epsilon^2}\right)$ iterations we can get \begin{small}
    \begin{align}
       \left \| {\boldsymbol{v}}^{(T)} - \left({\boldsymbol{I}} + \lambda \hat {\boldsymbol{L}}\right)^{-1} {\boldsymbol{u}} \right\|_{({\boldsymbol{I}} + \lambda \hat {\boldsymbol{L}})} \leq \epsilon \left\|\left({\boldsymbol{I}} + \lambda \hat {\boldsymbol{L}}\right)^{-1} {\boldsymbol{u}}\right\|_{({\boldsymbol{I}} + \lambda \hat {\boldsymbol{L}})}
    \end{align}\end{small}
    for any $\epsilon \in (0,1)$.
\end{theorem}
\begin{proof}
    Let ${\boldsymbol{H}} = {\boldsymbol{I}} + \lambda \hat {\boldsymbol{L}}$. Notice that the strongly convexity and Lipschitz constant of Problem \cref{eq:the-least-square} is $\mathscr C = 1$ and $L \leq 9\lambda^2$ respectively.
    From \Cref{lem:optimization}, take $\eta = \frac{1}{L}$ and ${\boldsymbol{v}}^{(0)} = {\boldsymbol{0}}$, we have \begin{align}
    \|{\boldsymbol{H}} {\boldsymbol{v}}^{(T)} - {\boldsymbol{u}}\|^2 \leq (1 - 9\lambda^2)^{T} \|{\boldsymbol{u}}\|^2
    \end{align}
    for any $T \in \mathbb N$.
    Therefore, \begin{align}
    \frac{\|{\boldsymbol{v}}^{(T)} - {\boldsymbol{H}}^{-1} {\boldsymbol{u}}\|^2_{{\boldsymbol{H}}}}{\|{\boldsymbol{H}}^{-1} {\boldsymbol{u}}\|^2_{{\boldsymbol{H}}}} & = \frac{\|{\boldsymbol{H}} {\boldsymbol{v}}^{(T)} - {\boldsymbol{u}}\|^2_{{\boldsymbol{H}}^{-1}}}{\|{\boldsymbol{u}}\|^2_{{\boldsymbol{H}}^{-1}}}
    \\ & \leq 3\lambda \frac{\|{\boldsymbol{H}} {\boldsymbol{v}}^{(T)} - {\boldsymbol{u}}\|_2^2}{\|{\boldsymbol{u}}\|_2^2}
    \\ & \leq 3\lambda (1- 9\lambda^2)^{{T}}.
    \end{align}
    To obtain an error rate $\epsilon$, the smallest number of iterations $T$ needed is \begin{align}
        T = \left(\log \frac{1}{1 - (9\lambda^2)^{-1}}\right)^{-1} \log \frac{{3\lambda}}{\epsilon^2} + 1 = O\left[ \left(\log \frac{1}{1 - \lambda^{-2}}\right)^{-1} \log \frac{\kappa}{\epsilon^2}\right] = O\left(\lambda^2 \log \frac{\lambda}{\epsilon^2}\right).
    \end{align}
\end{proof}

\paragraph{Overall Running Time.}

As we proved above in \Cref{thm:inner-analysis-naive}, the number of iterations needed for solving the inner problem to error rate $\beta$ with the implementation of \cite{twirls} is $\tilde O\left(\lambda^2 \log \frac{1}{\beta}\right)$, which is related to the hyper-parameter $\lambda$. For each inner iteration (i.e. \cref{eq:naive-inner-iter}), we need to compute $({\boldsymbol{I}} + \lambda {\boldsymbol{L}}) ({\boldsymbol{X}} {\boldsymbol{w}} - \hat {\boldsymbol{y}})$. Since this is a sparse matrix multiplication, the complexity of this step is $O(m + nd)$. Putting things together, we have the time complexity of calling inner problem solver is $\tilde O\left( (m+nd) \lambda^2 \log \frac{1}{\beta}\right)$.

From \Cref{lem:outer-analysis-naive}, the number of outer iterations needed is $\tilde O\left( \kappa_o \log\frac{1}{\epsilon}\right)$, where $\kappa_o$ is the condition number of the outer problem. Moreover, \Cref{lem:outer-analysis-naive} also indicates that we require $\beta \leq \frac{\epsilon^{1/2}}{25 \kappa({\boldsymbol{X}})\lambda}$, so solving inner problem takes $\tilde O(\lambda^2 \log 1/\beta) = \tilde O(\lambda^2 \log 1/\epsilon)$. In each outer iteration, we need to call inner problem solver constant times, and as well 
as performing constant matrix vector multiplications whose complexity is $O(nd)$. Therefore the overall running time of solving the training problem is $\tilde O\left[ \kappa_o \left(\lambda^2 (m + nd) + nd\right) \left(\log \frac{1}{\epsilon} \right)^2 \right]$.

\section{Proof of Theoretical Results}

We first note that since $\hat {\boldsymbol{L}}$ is normalized, all the eigenvalues of $\hat {\boldsymbol{L}}$ are in the range $[0,2)$. Therefore, we have $\sigma_{\min}({\boldsymbol{I}} + \lambda \hat{\boldsymbol{L}}) = 1$ and $\sigma_{\max}({\boldsymbol{I}} + \lambda \hat{\boldsymbol{L}}) \leq 1 + 2\lambda$. In the whole paper we view $\lambda$ as a large value (i.e. $\lambda \gg 1$). In practice, it is possible to use a small $\lambda$, in which case the algorithm works no worse than the case where $\lambda = 1$. Therefore the $\lambda$ used in the paper should actually be understood as $\max\{\lambda ,1\}$. With this assumption, below we assume $\sigma_{\max}({\boldsymbol{I}} + \lambda \hat{\boldsymbol{L}}) \leq 3\lambda$ for convenience.

\subsection{Descent Lemma}

In this subsection, we introduce a descent lemma that we will use to analyze the convergence rate. This is a standard result in convex optimization, and we refer interested readers to \citep{large-scale-optimization} for more details.

\begin{lemma}\label{lem:optimization}
    If $\ell: \mathbb R^d \to \mathbb R$ is $L$-Lipschitz smooth and $c$-strongly convex, and we have a sequence of points $\left\{{\boldsymbol{w}}^{(t)}\right\}_{t = 1}^T$ in $\mathbb R^d$ such that \begin{align}
        {\boldsymbol{w}}^{(t+1)} = {\boldsymbol{w}}^{(t)} - \eta {\boldsymbol{g}}^{(t)},
    \end{align}
    where $\left\|{\boldsymbol{g}}^{(t)} - \nabla \ell\left({\boldsymbol{w}}^{(t)}\right)\right\|_2 \leq \gamma \left\|\nabla \ell\left({\boldsymbol{w}}^{(t)}\right)\right\|_2$ and $\gamma < 1$, $\eta = \frac{1-\gamma}{(1 + \gamma)^2L}$, then we have \begin{align}
        \ell\left( {\boldsymbol{w}}^{(T)} \right) - \ell^* \leq \left[1 - \kappa^{-1} \left(\frac{1-\gamma}{1 + \gamma}\right)^2\right]^T \left[\ell\left({\boldsymbol{w}}^{(0)}\right) - \ell^*\right],
    \end{align}
    where $\ell^* = \inf_{{\boldsymbol{w}} \in \mathbb R^d} \ell({\boldsymbol{w}})$ and $\kappa = \frac{L}{c}$.
\end{lemma}
\begin{proof}

From the condition $\left\|{\boldsymbol{g}}^{(t)} - \nabla \ell\left({\boldsymbol{w}}^{(t)}\right)\right\|_2 \leq \gamma \left\|\nabla \ell\left({\boldsymbol{w}}^{(t)}\right)\right\|_2$, we have that \begin{align}
(1-\gamma) \|{\boldsymbol{g}}^{(t)}\|_2 \leq \left\|\nabla \ell\left({\boldsymbol{w}}^{(t)}\right)\right\|_2 \leq (1+\gamma) \|{\boldsymbol{g}}^{(t)}\|_2
\end{align}
and \begin{align}
-2\left<g^{(t)}, \nabla \ell\left({\boldsymbol{w}}^{(t)}\right)\right> & = \left\|{\boldsymbol{g}}^{(t)} - \nabla \ell\left({\boldsymbol{w}}^{(t)}\right)\right\|_2^2 - \left\|{\boldsymbol{g}}^{(t)}\right\|_2^2 - \left\|\nabla \ell\left({\boldsymbol{w}}^{(t)}\right)\right\|_2^2
\\ & \leq \left[ (\gamma^2 - 1) - (1-\gamma)^2\right] \left\|\nabla \ell\left({\boldsymbol{w}}^{(t)}\right)\right\|_2^2
\\ & =  2(\gamma-1) \left\|\nabla \ell\left({\boldsymbol{w}}^{(t)}\right)\right\|_2^2
\end{align}

    From Lipschitz smoothness, we have \begin{align}
        \ell\left({\boldsymbol{w}}^{(t+1)}\right) - \ell\left({\boldsymbol{w}}^{(t)}\right) & \leq -\eta \left< {\boldsymbol{g}}^{(t) } , \nabla \ell\left({\boldsymbol{w}}^{(t)}\right)\right> + \frac{1}{2} L \eta^2  \|{\boldsymbol{g}}^{(t)}\|_2^2
        \\ & \leq \left( \eta(\gamma-1) + \frac{1}{2} L \eta^2 (1 + \gamma)^2 \right)\left\|\nabla \ell\left({\boldsymbol{w}}^{(t)}\right)\right\|_2^2. \label{eq:descent}
    \end{align}
    It's not hard to show that the optimal $\eta$ for \cref{eq:descent} is $\eta = \frac{1-\gamma}{L(1 + \gamma)^2}$. Substituting $\eta = \frac{1-\gamma}{L(1 + \gamma)^2}$ to \cref{eq:descent} and use convexity we can get \begin{align}
        \ell\left({\boldsymbol{w}}^{(t+1)}\right) - \ell^* & \leq \left[\ell\left({\boldsymbol{w}}^{(t)}\right) - \ell^*\right]-\frac{(1-\gamma)^2}{2L(1 + \gamma)^2} \left\|\nabla \ell\left({\boldsymbol{w}}^{(t)}\right)\right\|_2^2
        \\ & \leq \left[\ell\left({\boldsymbol{w}}^{(t)}\right) - \ell^*\right] -\frac{c(1-\gamma)^2}{L(1 + \gamma)^2} \left[\ell\left({\boldsymbol{w}}^{(t)}\right) - \ell^*\right]
    \end{align}
    By induction we have \begin{align}
        \ell\left({\boldsymbol{w}}^{(T)}\right) - \ell^* & \leq \left[ 1 - \kappa^{-1} \left(\frac{1-\gamma}{1+\gamma}\right)^2 \right] ^T\left[\ell\left({\boldsymbol{w}}^{(0)}\right) - \ell^*\right].
    \end{align}
\end{proof}

\subsection{Bound of Loss Value by Gradient}

In this subsection, we derive an inequality that allows us to bound the value of loss function by the norm of gradient.

\begin{lemma}\label{lem:bound-loss-by-grad}
    If $\ell: \mathbb R^d \to \mathbb R$ is a $\mathscr C$-strongly convex and smooth function, and ${\boldsymbol{w}}^* = \arg\min_{{\boldsymbol{w}} \in \mathbb R^d} \ell({\boldsymbol{w}})$ is a global optimal, then for any ${\boldsymbol{w}} \in \mathbb R^d$ and $\epsilon \in (0,1)$, we have \begin{align}
        \ell({\boldsymbol{w}}) \leq \max\left\{ (1+\epsilon) \ell({\boldsymbol{w}}^*), \frac{1}{\epsilon \mathscr C}\|\nabla \ell({\boldsymbol{w}})\|^2 \right\}.
    \end{align}
\end{lemma}
\begin{proof}
    Using the inequality (4.12) from \citep{large-scale-optimization}, we have \begin{align}
    \ell({\boldsymbol{w}}) \leq \frac{1}{2\mathscr C}\|\nabla \ell({\boldsymbol{w}})\|^2 + \ell({\boldsymbol{w}}^*).
    \end{align}
    If $\ell({\boldsymbol{w}}) \geq (1+\epsilon) \ell({\boldsymbol{w}}^*)$, we have \begin{align}
        \ell({\boldsymbol{w}}) \leq \frac{1}{2\mathscr  C}\|\nabla \ell({\boldsymbol{w}})\|^2 + \frac{1}{1+\epsilon}\ell({\boldsymbol{w}}). \label{eq:C.2.a}
    \end{align}
    Shifting the terms in \cref{eq:C.2.a} gives $\ell({\boldsymbol{w}}) \leq \frac{1+\epsilon}{2\mathscr C\epsilon} \|\nabla \ell({\boldsymbol{w}})\|^2 \leq \frac{1}{\epsilon \mathscr C} \|\nabla \ell({\boldsymbol{w}})\|^2$, which proves the claim.
    
\end{proof}

\subsection{Proof of \Cref{lem:outer-analysis-naive}}\label{sec:proof-outer-analysis}

Given a point ${\boldsymbol{w}} \in \mathbb R^d$, we first consider the error between the estimated gradient $\widetilde{\nabla \ell_i({\boldsymbol{w}})}$ and the true gradient $\nabla \ell_i({\boldsymbol{w}})$. Below we denote ${\boldsymbol{H}} = {\boldsymbol{I}} + \lambda \hat {\boldsymbol{L}}$ and ${\boldsymbol{z}} = {\boldsymbol{X}}{\boldsymbol{w}} - \hat {\boldsymbol{y}}_i$. It's not hard to notice that $\ell_i({\boldsymbol{w}}) = \frac{1}{2} \left\|{\boldsymbol{H}}^{-1} {\boldsymbol{z}}\right\|^2$ and $\nabla \ell_i({\boldsymbol{w}}) = {\boldsymbol{X}}^\top {\boldsymbol{H}}^{-2} {\boldsymbol{z}}$. Moreover, we denote $\ell^*_i = \inf_{{\boldsymbol{w}}} \ell^*_i$ as the optimal value of $\ell_i$.

we have \begin{align}
         \left\| \widetilde{\nabla \ell_i({\boldsymbol{w}})} -  \nabla \ell_i({\boldsymbol{w}}) \right\| 
        = &  \left\| {\boldsymbol{X}}^\top\left[ \mathcal S\left( {\boldsymbol{H}}^{-1}{\boldsymbol{z}}\right) - {\boldsymbol{H}}^{-2} {\boldsymbol{z}}      
     +  \mathcal S( \mathcal S({\boldsymbol{z}})) - \mathcal S\left( {\boldsymbol{H}}^{-1}{\boldsymbol{z}}\right) \right] \right\|
     \\ \leq & \left\|{\boldsymbol{X}}^\top \left[\mathcal S\left( {\boldsymbol{H}}^{-1}{\boldsymbol{z}}\right) - {\boldsymbol{H}}^{-2} {\boldsymbol{z}}  \right]  \right\|  + \left\|{\boldsymbol{X}}^\top \left[\mathcal S( \mathcal S({\boldsymbol{z}}) - {\boldsymbol{H}}^{-1}{\boldsymbol{z}})  \right] \right\|.
    \end{align}

As shown, the gradient error can be decomposed into two terms, namely $\left\| \widetilde{\nabla \ell_i({\boldsymbol{w}})} -  \nabla \ell_i({\boldsymbol{w}}) \right\|  = E_1 + E_2$, where { $E_1 = \left\|{\boldsymbol{X}}^\top \left[\mathcal S\left({\boldsymbol{H}}^{-1}{\boldsymbol{z}}\right) - {\boldsymbol{H}}^{-2} {\boldsymbol{z}}  \right]  \right\|$ and  $E_2 = \left\|{\boldsymbol{X}}^\top \left[\mathcal S( \mathcal S( {\boldsymbol{z}}) - {\boldsymbol{H}}^{-1}{\boldsymbol{z}})  \right] \right\|$. Below we analyze each part separately.

Notice that, as we assumed ${\boldsymbol{X}}$ is full-rank, $\ell_i$ is a strongly convex function with parameter $\mathscr C = \sigma_{\min}\left( {\boldsymbol{X}}^\top {\boldsymbol{H}}^{-2} {\boldsymbol{X}} \right) \geq \frac{\sigma_{\min}({\boldsymbol{X}})^2}{9 \lambda^2}$. Let $\mathcal D = \left\{ \left. {\boldsymbol{w}} \in \mathbb R^{d}\right| \ell_i({\boldsymbol{w}}) \leq (1+\epsilon) \ell_i^*  \right\}$. If ${\boldsymbol{w}} \in \mathcal D$, then ${\boldsymbol{w}}$ is already a good enough solution. Below we assume ${\boldsymbol{w}} \not \in \mathcal D$, in which case by \Cref{lem:bound-loss-by-grad}, we have $\ell_i({\boldsymbol{w}}) \leq \frac{1}{\epsilon \mathscr C} \|\nabla \ell_i({\boldsymbol{w}})\|^2$, which means \begin{align}
 \left\| {\boldsymbol{H}}^{-1} {\boldsymbol{z}} \right\|^2 \leq \frac{2}{\mathscr C \epsilon } \left\| {\boldsymbol{X}}^{-1} {\boldsymbol{H}}^{-2} {\boldsymbol{z}} \right\|^2 \leq \frac{18 \lambda^2}{\sigma_{\min}({\boldsymbol{X}})^2 \epsilon}.
\end{align}

For $E_1$, we have \begin{align}
E_1 & = \left\|{\boldsymbol{X}}^\top \left[\mathcal S\left({\boldsymbol{H}}^{-1}{\boldsymbol{z}}\right) - {\boldsymbol{H}}^{-2} {\boldsymbol{z}}  \right]  \right\|_2
\\ & \leq \sigma_{\max}({\boldsymbol{X}}) \left\|\mathcal S\left( {\boldsymbol{H}}^{-1}{\boldsymbol{z}}\right) - {\boldsymbol{H}}^{-2} {\boldsymbol{z}}  \right\|_2
\\ & \leq \sigma_{\max}({{\boldsymbol{X}}}) \left\|\mathcal S\left( {\boldsymbol{H}}^{-1}{\boldsymbol{z}}\right) - {\boldsymbol{H}}^{-2} {\boldsymbol{z}}  \right\|_{{\boldsymbol{H}}}
\\ & \leq \sigma_{\max}({{\boldsymbol{X}}}) \mu \left\| {\boldsymbol{H}}^{-2} {\boldsymbol{z}}  \right\|_{{\boldsymbol{H}}}
\\ & \leq \sigma_{\max}({{\boldsymbol{X}}}) \mu \left\| {\boldsymbol{H}}^{-1} {\boldsymbol{z}}  \right\|_2
\\ & \leq  \sqrt{ \frac{18}{\epsilon}}\kappa ({{\boldsymbol{X}}}) \lambda \mu \left\|  {\boldsymbol{X}}^\top {\boldsymbol{H}}^{-2} {\boldsymbol{z}}  \right\|_2
\\ & \leq 5 \epsilon^{-1/2} \kappa({{\boldsymbol{X}}}) \lambda \mu \|\nabla \ell_i({\boldsymbol{w}})\|_2.
\end{align}
For $E_2$, we have \begin{align}
E_2 & = \left\|{\boldsymbol{X}}^\top \left[\mathcal S(\mathcal S( {\boldsymbol{z}}) - {\boldsymbol{H}}^{-1}{\boldsymbol{z}})  \right] \right\|_2
\\ & \leq {\sigma_{\max}({\boldsymbol{X}})} ( 1 + \mu)  \|{\boldsymbol{H}}^{-1} \left(\mathcal S({\boldsymbol{z}}) - {\boldsymbol{H}}^{-1} {\boldsymbol{z}}\right)\|_{{\boldsymbol{H}}}
\\ & \leq {\sigma_{\max}({\boldsymbol{X}})} ( 1 + \mu)  \|\mathcal S({\boldsymbol{z}}) - {\boldsymbol{H}}^{-1} {\boldsymbol{z}}\|_{{\boldsymbol{H}}}
\\ & \leq {\sigma_{\max}({\boldsymbol{X}})} ( 1 + \mu)\mu  \|{\boldsymbol{H}}^{-1} {\boldsymbol{z}}\|_{{\boldsymbol{H}}}
\\ & \leq \sqrt{ \frac{18}{\epsilon}} \kappa({\boldsymbol{X}}) ( 1 + \mu)\mu  \lambda\sqrt{3\lambda} \|{\boldsymbol{X}}^\top{\boldsymbol{H}}^{-2} {\boldsymbol{z}}\|_2
\\ & \leq 20 \epsilon^{-1/2} \kappa({\boldsymbol{X}})  \mu \lambda^2 \|\nabla \ell_i({\boldsymbol{w}}_i)\|_2.
\end{align}

Combine the bound of $E_1$ and $E_2$, we have \begin{align}
\left\|\widetilde{\nabla \ell_i({\boldsymbol{w}}_i)} - \nabla \ell_i({\boldsymbol{w}}_i)\right\| \leq E_1 + E_2 \leq 25 \epsilon^{-1/2} \kappa({\boldsymbol{X}}) \mu \lambda^2 \|\nabla \ell_i({\boldsymbol{w}}_i)\|_2.
\end{align}

Now, consider the optimization. Let $\left\{{\boldsymbol{w}}_i^{(t)}\right\}_{t=1}^T$ be a sequence such that \begin{align}
{\boldsymbol{w}}_i^{(t+1)} = {\boldsymbol{w}}^{(t)} - \eta \widetilde{\nabla \ell_i\left({\boldsymbol{w}}_i^{(t)}\right)}.
\end{align}

Let $\kappa$ be the condition number of ${\boldsymbol{X}}^\top {\boldsymbol{H}}^{-2} {\boldsymbol{X}}$ and let $\ell_i^*$ be the optimal value of $\ell_i$. Let $\gamma = 25 \epsilon^{-1/2}\kappa({\boldsymbol{X}})\lambda^2\mu \leq \frac{1}{2}$. From \Cref{lem:optimization}, we have with a proper value of $\eta$,
\begin{align}
\ell_i\left({\boldsymbol{w}}_i^{(T)}\right) - \ell^* & \leq \left(1 - \left(\frac{1-\gamma}{1+\gamma}\right)^2\kappa^{-1} \right)^T\left[\ell_i\left({\boldsymbol{w}}_i^{(0)}\right) - \ell_i^*\right]
\\ & \leq \left(1 - (9\kappa)^{-1}\right)^T \left[\ell_i\left({\boldsymbol{w}}_i^{(0)}\right) - \ell_i^*\right] .
\end{align}

Therefore, the number of iterations $T$ required to achieve $\epsilon$ error rate is \begin{align}
T = O\left( \left(\log \frac{1}{1 - (9\kappa)^{-1}}\right)^{-1} \log \frac{1}{\epsilon} \right) = O\left( \kappa \log \frac{1}{\epsilon} \right).
\end{align}

\subsection{Proof of \Cref{lem:regularized-spectral-sprsification}}
The basic idea of \Cref{alg:sparse} is to use ridge leverage score sampling methods to obtain a spectral sparsifier.
  Let $\hat {\boldsymbol{B}} = {\boldsymbol{B}}{\boldsymbol{D}}^{-1/2}$ be the normalized incidence matrix and $\hat {\boldsymbol{b}}_i$ the $i$-th row of $\hat {\boldsymbol{B}}$.  Given $\lambda^{-1} > 0$, for $i\in \{1,2,\cdots, m\}$, the $i$-th ridge leverage score is defined as
\begin{align}
l_i := & ~ \hat{\boldsymbol{b}}_i^\top (\hat{\boldsymbol{L}} + \lambda^{-1}\boldsymbol{I})^{-1} \hat{\boldsymbol{b}}_i \\
= & ~ \hat{\boldsymbol{b}}_i^\top (\hat{\boldsymbol{L}} + \lambda^{-1}\boldsymbol{I})^{-1} \hat{\boldsymbol{L}}(\hat{\boldsymbol{L}} + \lambda^{-1}\boldsymbol{I})^{-1} \hat{\boldsymbol{b}}_i +\lambda^{-1}\hat{\boldsymbol{b}}_i^\top(\hat{\boldsymbol{L}} + \lambda^{-1}\boldsymbol{I})^{-2} \hat{\boldsymbol{b}}_i \\
= & ~ \|\hat{\boldsymbol{B}}(\hat{\boldsymbol{L}} + \lambda^{-1}\boldsymbol{I})^{-1}\hat{\boldsymbol{b}}_i\|^2 + \lambda^{-1}\|(\hat{\boldsymbol{L}} + \lambda^{-1}\boldsymbol{I})^{-1}\hat{\boldsymbol{b}}_i\|^2.
\end{align}

It is not affordable to compute all the $m$ ridge leverage scores exactly. Therefore, we  first use Johnson–Lindenstrauss lemma to reduce the dimension. Recall that in \Cref{alg:sparse} we define $\boldsymbol{\Pi}_1\in\R^{k \times m}$ and $\boldsymbol{\Pi}_2 \in \R^{k\times n}$ to be Gaussian sketches (that is, each entry of the matrices are i.i.d Gaussian random variables $\mathcal{N}(0,1/k)$). We set $k=O(\log m)$, and by using \Cref{lem:JL} we have the following claim holds with probability at least $1 - \frac{1}{8n}$:
\begin{align}
\|\boldsymbol{\Pi}_1 \hat{\boldsymbol{B}}(\hat{\boldsymbol{L}} + \lambda^{-1}\boldsymbol{I})^{-1}\hat{\boldsymbol{b}}_i\| \approx_{2^{1/4}-1} \| \hat{\boldsymbol{B}}(\hat{\boldsymbol{L}} + \lambda^{-1}\boldsymbol{I})^{-1}\hat{\boldsymbol{b}}_i\| ~~~\text{for all $i\in \{1,2,\cdots, m\}$}.
\end{align}
Similarly,the following claim also holds with probability at least $1 - \frac{1}{8n}$:
\begin{align}
\|\boldsymbol{\Pi}_2(\hat{\boldsymbol{L}} + \lambda^{-1}\boldsymbol{I})^{-1}\hat{\boldsymbol{b}}_i\| \approx_{2^{1/4}-1} \|(\hat{\boldsymbol{L}} + \lambda^{-1}\boldsymbol{I})^{-1}\hat{\boldsymbol{b}}_i\|~~~\text{for all $i \in \{1,2,\cdots, m\}$}.
\end{align}

By summing over above two inequalities (after squared) and taking an union bound, with probability $1 - \frac{1}{4n}$ we have
\begin{align}\label{eq:approx_1}
\|\boldsymbol{\Pi}_1\hat{\boldsymbol{B}}(\hat{\boldsymbol{L}} + \lambda^{-1}\boldsymbol{I})^{-1}\hat{\boldsymbol{b}}_i\|^2 + \lambda^{-1}\|\boldsymbol{\Pi}_2(\hat{\boldsymbol{L}} + \lambda^{-1}\boldsymbol{I})^{-1}\hat{\boldsymbol{b}}_i\|^2\approx_{\sqrt{2}-1} l_i ~~~\text{for all}~~~ i\in \{1,2,\cdots, m\}.
\end{align}
However it is still too expensive to compute $\boldsymbol{\Pi}_1\hat{\boldsymbol{B}}(\hat{\boldsymbol{L}} + \lambda^{-1}\boldsymbol{I})^{-1}$ and $\boldsymbol{\Pi}_2(\hat{\boldsymbol{L}} + \lambda^{-1}\boldsymbol{I})^{-1}$, since computing $(\hat{\boldsymbol{L}} + \lambda^{-1}\boldsymbol{I})^{-1}$ itself takes prohibitive $O(n^3)$ time. Instead, notice that $\hat{\boldsymbol{L}} + \lambda^{-1}\boldsymbol{I}$ is a SDD matrix, thus we can apply the SDD solver to the $k$ columns of matrix $(\boldsymbol{\Pi}_1\hat{\boldsymbol{B}})^\top$ and $(\boldsymbol{\Pi}_2)^\top$ respectively. 
According to \Cref{lem:sdd-solver}, there is a linear operator $\mathsf{SDDSolver}_{\epsilon}({\hat{\boldsymbol{L}}+\lambda^{-1}\boldsymbol{I}};\boldsymbol{x})$ that runs in time $\tilde{O}(\mathrm{nnz}({\hat{\boldsymbol{L}}+\lambda^{-1}\boldsymbol{I}}) \cdot\log 1/\epsilon) = \tilde{O}(m \log1/\epsilon)$ such that for any $\boldsymbol{x}^\top \in \R^{n}$, it outputs $\tilde{\boldsymbol{x}}$ that satisfies $\|\tilde{\boldsymbol{x}} - \boldsymbol{x}^\top({\hat{\boldsymbol{L}}+\lambda^{-1}\boldsymbol{I}})^{-1}\|_{{\hat{\boldsymbol{L}}+\lambda^{-1}\boldsymbol{I}}} \leq\epsilon\|\boldsymbol{x}^\top({\hat{\boldsymbol{L}}+\lambda^{-1}})\|_{{\hat{\boldsymbol{L}}+\lambda^{-1}\boldsymbol{I}}}$. For our purpose we set $\epsilon = 2^{1/4}$. Denote $\boldsymbol{B}_{\mathcal{S}}$ as the matrix obtained by applying each column of $(\boldsymbol{\Pi}_1\hat{\boldsymbol{B}})^\top$ to $\mathsf{SDDSolver}_{\epsilon}({\hat{\boldsymbol{L}}+\lambda^{-1}\boldsymbol{I}};\boldsymbol{x})$, that is, the $j$-th row of $\mathsf{SDDSolver}_{2^{1/4}}({\hat{\boldsymbol{L}}+\lambda^{-1}\boldsymbol{I}};(\boldsymbol{\Pi}_1\hat{\boldsymbol{B}})^\top_j)$. Similarly we denote $\boldsymbol{\Pi}_{\mathcal{S}}$ as the matrix with $j$-th row equals to $\mathsf{SDDSolver}_{2^{1/4}}({\hat{\boldsymbol{L}}+\lambda^{-1}\boldsymbol{I}};(\boldsymbol{\Pi}_2)^\top_j)$. By using Lemma~\ref{lem:sdd-solver} to both solvers we have
\begin{align}\label{eq:approx_2}
\|\boldsymbol{B}_{\mathcal{S}} \hat{\boldsymbol{b}}_i\|^2 + \lambda^{-1} \|\boldsymbol{\Pi}_{\mathcal{S}} \hat{\boldsymbol{b}}_i\|^2 \approx_{\sqrt{2}-1} \|\boldsymbol{\Pi}_1\hat{\boldsymbol{B}}(\hat{\boldsymbol{L}} + \lambda^{-1}\boldsymbol{I})^{-1}\hat{\boldsymbol{b}}_i\|^2 + \lambda^{-1}\|\boldsymbol{\Pi}_2(\hat{\boldsymbol{L}} + \lambda^{-1}\boldsymbol{I})^{-1}\hat{\boldsymbol{b}}_i\|^2 ~~~\text{for all}~~~ i\in \{1,2,\cdots, m\}.
\end{align}
By \Cref{eq:approx_1} and \Cref{eq:approx_2}, if we set $\tilde{l}_i := \|\boldsymbol{B}_{\mathcal{S}} \hat{\boldsymbol{b}}_i\|^2 + \lambda^{-1} \|\boldsymbol{\Pi}_{\mathcal{S}}\hat{\boldsymbol{b}}_i\|^2$, then with probability $1 - \frac{1}{4n}$, we obtain all the approximation of ridge leverage scores $\{\tilde{l}_i\}_{i=1}^m$ such that $\tilde{l}_i \approx_{1/2} l_i$ holds for all $i$.
Notice that since $\boldsymbol{B}_{\mathcal{S}}, \boldsymbol{\Pi}_{\mathcal{S}} \in \R^{k\times n}$, and that each $\hat{\boldsymbol{b}}_i$ only contains $2$ non-zero entries, thus it takes $\tilde{O}(k \cdot m)$ to pre-compute $\boldsymbol{B}_{\mathcal{S}}$ and $\boldsymbol{\Pi}_{\mathcal{S}}$, and takes $O(2k \cdot m)$ to compute $\{\boldsymbol{B}_{\mathcal{S}} \hat{\boldsymbol{b}}_i, \boldsymbol{\Pi}_{\mathcal{S}} \hat{\boldsymbol{b}}_i\}_{i=1}^m$. To summarize, computing all $\tilde{l}_i$ takes $\tilde{O}(km) = \tilde{O}(m\log m) = \tilde{O}(m)$.
With these ridge leverage score approximations, we apply \Cref{lem:rls_sampling} to matrix $\hat{\boldsymbol{B}}$ with choice $\delta = 1/4n$. By setting $\tilde{\boldsymbol{L}} := \hat{\boldsymbol{B}}^\top \boldsymbol{S}^\top\boldsymbol{S}\hat{\boldsymbol{B}}$ we have $\tilde{\boldsymbol{L}} + \lambda^{-1} \boldsymbol{I} \approx_{\epsilon} \hat{\boldsymbol{L}} + \lambda^{-1} \boldsymbol{I}$ holds with probability $1 - 1/4n$. By applying another union bound we obtain our final result:
\begin{align}
\tilde{\boldsymbol{L}} + \lambda^{-1} \boldsymbol{I} \approx_{\epsilon} \hat{\boldsymbol{L}} + \lambda^{-1} \boldsymbol{I} ~~~\text{with probability}~~~ 1 - \frac{1}{2n}.
\end{align}
Finally, according to \Cref{lem:rls_sampling}, the number of edges of $\tilde{\boldsymbol{L}}$ is $s = C n_{\lambda} \log (n^2) / \epsilon^2 = O(n_{\lambda} \log n / \epsilon^2)$. Since the last step of computing $\tilde{\boldsymbol{L}} = (\boldsymbol{S}\hat{\boldsymbol{B}})^\top(\boldsymbol{S}\hat{\boldsymbol{B}})$ only takes $O(s) = \tilde{O}(n_{\lambda}/\epsilon^2)$ due to the sparsity of $\hat{\boldsymbol{B}}$, the overall time complexity of \Cref{alg:sparse} is $\tilde{O}(m+n_{\lambda}/\epsilon^2)$.

\begin{lemma}[Johnson–Lindenstrauss, \cite{jl-gaussian}]\label{lem:JL}
Let $\boldsymbol{\Pi} \in \R^{k \times n}$ be Gaussian sketch matrix with each entry independent and equal to $\mathcal{N}(0,1/k)$ where $\mathcal{N}(0,1)$ denotes a standard Gaussian random variable. If we choose $k = O\left(\frac{\log(1/\delta)}{\epsilon^2}\right)$, then for any vector $\boldsymbol{x}\in\R^n$, with probability $1-\delta$ we have
\begin{align}
(1-\epsilon)\|\boldsymbol{x}\| \leq \|\boldsymbol{\Pi} \boldsymbol{x}\| \leq (1+\epsilon)\|\boldsymbol{x}\|.
\end{align}
\end{lemma}

\begin{lemma}[Spectral approximation, \cite{ridge_score}]\label{lem:rls_sampling}
Let $\boldsymbol{S}$ be an $s \times m$ subsampling matrix with probabilities $p_i = \tilde{\ell}_i / Z$ where $\tilde{\ell}_i \approx_{1/2} \ell_i$ and $Z$ is the normalization constant. If we have $s \geq C n_{\lambda} \log(n/\delta) / \epsilon^2$ for some constant $C > 0$ and $\epsilon, \delta \in (0,1/2]$, then we have
\begin{align}
\hat{\boldsymbol{B}}^\top \boldsymbol{S}^\top\boldsymbol{S}\hat{\boldsymbol{B}} + \lambda^{-1}\boldsymbol{I} \approx_{\epsilon} \hat{\boldsymbol{B}}^\top\hat{\boldsymbol{B}}+\lambda^{-1}\boldsymbol{I}
\end{align}
holds with probability $1-\delta$. Here $n_{\lambda} = \mathrm{Tr} [ \hat {\boldsymbol{L}}  (\hat {\boldsymbol{L}} + \lambda^{-1}{\boldsymbol{I}})^{-1} ]$.
\end{lemma}


\subsection{Proof of \Cref{lem:preconditioner}}\label{sec:proof-lem-precond}

We first prove a lemma which is related to the approximation rate of squared matrices.
\begin{lemma}\label{lem:square}
    Suppose that ${\boldsymbol{\Sigma}}$ and $\widetilde {\boldsymbol{\Sigma}}$ are two $n \times n$ PD matrices, and $\widetilde {\boldsymbol{\Sigma}} \approx_{\frac{\beta}{\kappa}} {\boldsymbol{\Sigma}}$, where $\kappa$ is the condition number of ${\boldsymbol{\Sigma}}$ and $\beta \in \left(0, \frac 18\right)$, then we have \begin{align}\widetilde {\boldsymbol{\Sigma}}^2 \approx_{8\beta} {\boldsymbol{\Sigma}}^2.\end{align}
\end{lemma}
\begin{proof}

    Let $\epsilon = \frac{\beta}{\kappa}$. 
    The condition ${\boldsymbol{\Sigma}} \approx_{\epsilon} \widetilde {\boldsymbol{\Sigma}}$ implies that $\left\| {\boldsymbol{\Sigma}} - \widetilde{\boldsymbol{\Sigma}} \right\| \leq \epsilon \left\|{\boldsymbol{\Sigma}}\right\|$. We have \begin{align}
        \left\| \left({\boldsymbol{\Sigma}} - \widetilde {\boldsymbol{\Sigma}}\right) {\boldsymbol{x}} \right\| & = \left\| \left({\boldsymbol{\Sigma}} - \widetilde {\boldsymbol{\Sigma}}\right) {\boldsymbol{\Sigma}}^{-1} {\boldsymbol{\Sigma}} {\boldsymbol{x}}\right\|
        \\ & \leq \left\| \left({\boldsymbol{\Sigma}} - \widetilde {\boldsymbol{\Sigma}}\right) {\boldsymbol{\Sigma}}^{-1} \right\| \times \left\| {\boldsymbol{\Sigma}} {\boldsymbol{x}}\right\|
        \\ & \leq \left\| \left({\boldsymbol{\Sigma}} - \widetilde {\boldsymbol{\Sigma}}\right) \right\| \left\|{\boldsymbol{\Sigma}}^{-1} \right\| \times \left\| {\boldsymbol{\Sigma}} {\boldsymbol{x}}\right\|
        \\ & \leq  \epsilon \left\| {\boldsymbol{\Sigma}}\right\| \left\| {\boldsymbol{\Sigma}}^{-1}\right\| \times \left\|{\boldsymbol{\Sigma}} {\boldsymbol{x}}\right\|
        \\ & = \epsilon \kappa \left\| {\boldsymbol{\Sigma}} {\boldsymbol{x}}\right\|
        \\ & = \beta \left\| {\boldsymbol{\Sigma}} {\boldsymbol{x}}\right\|.
    \end{align}
    
    For any ${\boldsymbol{x}} \in \mathbb R^n$, by triangle inequality we have \begin{align}
        \left\|{\boldsymbol{\Sigma}} {\boldsymbol{x}}\right\| - \left\|\left({\boldsymbol{\Sigma}} - \widetilde {\boldsymbol{\Sigma}}\right){\boldsymbol{x}}\right\| & \leq \left\|\widetilde {\boldsymbol{\Sigma}}{\boldsymbol{x}}\right\|
        \leq \left\|{\boldsymbol{\Sigma}} {\boldsymbol{x}}\right\| + \left\|\left({\boldsymbol{\Sigma}} - \widetilde {\boldsymbol{\Sigma}}\right){\boldsymbol{x}}\right\| \label{eq:square-lemma-triangle}.
    \end{align}
    Subtracting the inequality $\left\|\left({\boldsymbol{\Sigma}} - \widetilde {\boldsymbol{\Sigma}}\right){\boldsymbol{x}}\right\| \leq \beta \left\|{\boldsymbol{\Sigma}} {\boldsymbol{x}}\right\|$ we derived before into \cref{eq:square-lemma-triangle} and squaring all sides, we have \begin{align}
        (1 - \beta)^2 {\boldsymbol{x}}^\top {\boldsymbol{\Sigma}}^2  {\boldsymbol{x}} \leq {\boldsymbol{x}}^\top \widetilde {\boldsymbol{\Sigma}}  {\boldsymbol{x}}\leq (1 + \beta)^2 {\boldsymbol{x}}^\top {\boldsymbol{\Sigma}}^2 {\boldsymbol{x}}.
    \end{align}
    Since $0 < \beta < \frac{1}{8}$, the claim is proved.
\end{proof}

Let ${\boldsymbol{T}} ={\boldsymbol{X}}^\top \left({\boldsymbol{I}} + \lambda \hat {\boldsymbol{L}}\right) ^{-2} {\boldsymbol{X}}$ be the true Hessian. Let $\tilde {\boldsymbol{T}} = {\boldsymbol{X}}^\top \left({\boldsymbol{I}} + \lambda \tilde {\boldsymbol{L}}\right)^{-2} {\boldsymbol{X}}$ be the approximated Hessian using sparsified Laplacian $\tilde {\boldsymbol{L}}$. By \Cref{lem:regularized-spectral-sprsification}, ${\boldsymbol{I}} + \lambda \tilde {\boldsymbol{L}} \approx_{\frac{\beta}{3 \lambda}} {\boldsymbol{I}} + \lambda {\boldsymbol{L}}$ with probability at least $1 - \frac{1}{2n}$. From \Cref{lem:square}, we have $\tilde {\boldsymbol{T}} \approx_{8 \beta} {\boldsymbol{T}}$. 
        
Next, we show that ${\boldsymbol{Q}}^\top {\boldsymbol{Q}} \approx_{O(\beta)} {\boldsymbol{T}}$. Let $\tilde {\boldsymbol{H}} = {\boldsymbol{I}} + \lambda \tilde{ \boldsymbol{L}}$. Let $\mathcal S$ be the operator defined by $\mathsf{SDDSolver}_{\frac{\beta}{\sqrt{3 \lambda}}}\left(\tilde {\boldsymbol{H}}, \cdot\right)$, i.e. ${\boldsymbol{q}}_j = \mathcal S({\boldsymbol{x}}_j)$ where ${\boldsymbol{q}}_j$ and ${\boldsymbol{x}}_j$ are the $j$-th column of ${\boldsymbol{Q}}$ and ${\boldsymbol{X}}$ respectively. From \Cref{lem:sdd-solver} we have for any ${\boldsymbol{z}} \in \mathbb R^d$, \begin{align}
        \left\|{\boldsymbol{Q}} {\boldsymbol{z}} - \tilde{\boldsymbol{H}}^{-1} {\boldsymbol{X}}{\boldsymbol{z}}  \right\|_{2}^2  & = \left\| \sum_{j=1}^dz_j\left(\mathcal S({\boldsymbol{x}}_j) - \tilde {\boldsymbol{H}} {\boldsymbol{x}}_j\right) \right\|_2^2
        \\ & = \left\|\mathcal S\left(\sum_{j=1}^d z_j{\boldsymbol{x}}_j\right) - \tilde {\boldsymbol{H}}\left(\sum_{j=1}^d z_j {\boldsymbol{x}}_j\right)\right\|_2^2
        \\ & \leq \frac{\beta^2}{3\lambda}\left\| \tilde{\boldsymbol{H}}^{-1} {\boldsymbol{X}}{\boldsymbol{z}}  \right\|_{\tilde{\boldsymbol{H}}}^2
        \\ & \leq \beta^2 \|\tilde{\boldsymbol{H}}^{-1} {\boldsymbol{X}}{\boldsymbol{z}}\|_2^2.
\end{align}
    Therefore we have 
    \begin{align}
        (1-\beta)\|\tilde{\boldsymbol{H}}^{-1} {\boldsymbol{X}}{\boldsymbol{z}}\|_2 \leq \|{\boldsymbol{Q}}{\boldsymbol{z}}\| \leq (1 + \beta) \|\tilde{\boldsymbol{H}}^{-1} {\boldsymbol{X}}{\boldsymbol{z}}\|_2,    
    \end{align}
    and this is equivalent to \begin{align}
        (1-\beta)^2 {\boldsymbol{z}}^\top ({\boldsymbol{X}} \tilde{\boldsymbol{H}}^{-2} {\boldsymbol{X}}) {\boldsymbol{z}} \leq {\boldsymbol{z}}^\top {\boldsymbol{Q}}^\top {\boldsymbol{Q}} {\boldsymbol{z}} \leq (1 + \beta)^2 {\boldsymbol{z}}^\top \left({\boldsymbol{X}}^\top\tilde {\boldsymbol{H}}^{-2} {\boldsymbol{X}}\right){\boldsymbol{z}}.
    \end{align}
    Notice that ${\boldsymbol{X}}^\top \tilde {\boldsymbol{H}}^{-2} {\boldsymbol{X}} = \tilde {\boldsymbol{T}}$. We conclude that ${\boldsymbol{Q}}^\top {\boldsymbol{Q}} \approx_{2\beta + \beta^2} \tilde {\boldsymbol{T}}$. When $\beta < \frac{1}{8}$, we have ${\boldsymbol{Q}}^\top {\boldsymbol{Q}} \approx_{4\beta} \tilde {\boldsymbol{T}}$.

    Lastly, from \Cref{lem:SRHT} and the discussions in \Cref{sec:fast-matrix-multiplication}, we have there exists a constant $C' \in (0,1/4)$, such that \begin{align}
{\boldsymbol{P}} = \tilde {\boldsymbol{Q}}^\top \tilde {\boldsymbol{Q}} \approx_{C'} {\boldsymbol{Q}}^\top {\boldsymbol{Q}}
\end{align}
with probability at least $1 - \frac{1}{2n}$.

Put the results above together and we get \begin{align}{\boldsymbol{P}} \approx_{12\beta + C'} {\boldsymbol{T}}.\end{align} Notice that $12\beta + C' < \frac{1}{2}$. Using a union bound can prove that the fail probability of this whole process is bounded by $\frac{1}{n}$.

\subsection{Proof of \Cref{thm:after-preconditioning}}

Let ${\boldsymbol{F}} = ({\boldsymbol{I}} + \lambda \hat {\boldsymbol{L}}) {\boldsymbol{X}}$. We have the Hessian of $\ell'$ is \begin{align}
\nabla^2 \ell' = {\boldsymbol{P}}^{-\frac{1}{2}}  {\boldsymbol{F}}^\top {\boldsymbol{F}} {\boldsymbol{P}}^{-\frac{1}{2}}.
\end{align}

From the condition that ${\boldsymbol{P}} \approx_{c_0} {\boldsymbol{F}}^\top {\boldsymbol{F}}$, we have ${\boldsymbol{F}}^\top {\boldsymbol{F}}  \preceq (1 + c_0) {\boldsymbol{P}}$, which implies \begin{align}
{\boldsymbol{P}}^{-1/2} {\boldsymbol{F}}^\top {\boldsymbol{F}}  {\boldsymbol{P}}^{-1/2} \preceq (1 + c_0){\boldsymbol{I}}.
\end{align}

Similarly, since ${\boldsymbol{P}}\preceq (1 + c_0) {\boldsymbol{F}}^\top {\boldsymbol{F}}$, we have \begin{align}
\left({\boldsymbol{F}}^\top {\boldsymbol{F}}\right)^{-1} \preceq (1 + c_0) {\boldsymbol{P}}^{-1},
\end{align}
which implies
\begin{align}
\left[ \lambda_{\min}({\boldsymbol{P}}^{-1/2} {\boldsymbol{F}}^\top {\boldsymbol{F}}  {\boldsymbol{P}}^{-1/2})\right] ^{-1} & = \lambda_{\max}\left({\boldsymbol{P}}^{1/2} \left({\boldsymbol{F}}^\top {\boldsymbol{F}}\right)^{-1}  {\boldsymbol{P}}^{1/2}\right)
\\ & \leq 1 + c_0.
\end{align}.

Notice that $\ell'$ and $\ell$ have the same global optimal value, let it be $\ell^*$. For any ${\boldsymbol{w}}'$ satisfies $\ell'({\boldsymbol{w}}') - \ell^* \leq \gamma$, we have \begin{align}
\ell\left({\boldsymbol{P}}^{1/2}{\boldsymbol{w}}'\right) - \ell^* = \ell'({\boldsymbol{w}}') - \ell^*  \leq \gamma.
\end{align}
Therefore, if ${\boldsymbol{w}}'$ is a solution of $\ell'$ with $\epsilon$ error rate, then ${\boldsymbol{P}}^{1/2} {\boldsymbol{w}}'$ is a solution of $\ell$ with $\epsilon$ error rate.

\subsection{Proof of \Cref{thm:main}}\label{sec:proof-main}

As we noted in the main paper, \Cref{alg:herta} is composed by two components: constructing the preconditioner ${\boldsymbol{P}}$ and applying it to the optimization.

\paragraph{Constructing the Preconditioner.}~ From \Cref{lem:regularized-spectral-sprsification}, the first step that applies spectral sparsifier to get $\tilde {\boldsymbol{L}}$ requires $\tilde O(m)$ time and the number of non-zero entries in $\tilde {\boldsymbol{L}}$ is $\tilde O\left(n_{\lambda}  \lambda^2 \right)$. By \Cref{lem:sdd-solver}, running the SDD solver with error rate $\frac{\beta}{3 \lambda}$ in the second step requires $ \tilde O\left({n_{\lambda} \lambda^2} \right)$ time. Notice that since ${\boldsymbol{X}}$ is an $n \times d$ matrix, we actually need to run SDD solver for $d$ times, and this introduces another $d$ factor in the time complexity. 

As we noted in \Cref{sec:fast-matrix-multiplication}, applying the Hadamard transformation to ${\boldsymbol{R}} {\boldsymbol{Q}}$ requires $\tilde O(d)$ time, and applying the subsampling requires $\tilde O(n)$ time. Since $\tilde {\boldsymbol{Q}} \in \mathbb R^{s \times d}$, calculating ${\boldsymbol{P}} = \tilde {\boldsymbol{Q}}^\top {\boldsymbol{Q}}$ requires $O(d^2 s)$ time. Since $s = \tilde O(d)$, this step takes $\tilde O(d^3)$ time. We use brute force to calculate ${\boldsymbol{P}}' = {\boldsymbol{P}}^{-1/2}$, and this takes $O(d^3)$ time.

As a summary, constructing the preconditioner ${\boldsymbol{P}}'$ takes $\tilde O(n_\lambda \lambda^2 d + nd + d^3)$ time. Notice that we only perform this step once during the whole algorithm.

\paragraph{Performing the Iterations.}~
Next we consider the time required for each iteration. We calculate ${\boldsymbol{u}}^{(t)}$ from right to left and it takes $O(nd)$ time. Next, we need to perform two SDD solvers with error rate $\mu$. From \Cref{lem:sdd-solver}, it takes $\tilde O\left[ m \log\left(\frac{1}{\mu}\right)\right]$, which is $\tilde O\left[ m \log \left(\frac 1\epsilon\right)\right]$. Calculating ${\boldsymbol{g}}^{(t)}$ and ${\boldsymbol{w}}^{(t)}$ is straight-forward and takes $\tilde O(nd)$ time.

To conclude, performing each iteration requires $\tilde O\left[ m \log \left(\frac 1\epsilon \right) + nd\right]$ time. By \Cref{lem:outer-analysis-naive} and \Cref{lem:preconditioner},  with a proper step size, the number of iterations needed for solving outer problem is $\tilde O\left(\log 1/\epsilon\right)$.

Combining the analysis above together, to overall complexity is \begin{align}
\tilde O\left[ n_\lambda \lambda^2 d + nd + d^3 + \left (m \log\left(\frac 1\epsilon\right)+ nd\right) \log\left(\frac 1\epsilon\right) \right] = \tilde O\left(n_\lambda \lambda^2 d + d^3 + (m+nd) \left(\log 1/\epsilon\right)^2\right),
\end{align}
which proves the claim.

\section{Further Discussions}

In this section, we extend some of the discussions in the main paper.

\subsection{An Analysis of CE Loss v.s. MSE Loss}\label{sec:anlysis-of-ce-loss}

Although HERTA is derived from MSE loss, the experiment result shows it also works on CE loss. In this subsection we provide an analysis showing the similarity of the gradient and Hessian of TWIRLS on CE loss and MSE loss, to offer a intuitive explanation why a method that is derived from MSE loss can work on CE loss. 

In this section, for a node $u \in \{1,2,\cdots,n\}$, we use ${\boldsymbol{y}}^{(u)} \in \mathbb R^c$ to represent the $u$-th row of ${\boldsymbol{Y}}$ (notice we use super-script here to distinguish from ${\boldsymbol{y}}_i$ used before), ${\boldsymbol{h}}_u \in \mathbb R^n$ to represent the $u$-th row of $\left({\boldsymbol{I}}+ \lambda \hat {\boldsymbol{L}}\right)^{-1}$. For $p \in \{1,2,\cdots d\}$, we use ${\boldsymbol{x}}_p \in \mathbb R^n$ to denote the $p$-th column of ${\boldsymbol{X}}$. For $i \in \{1,2,\cdots c\}$, we use $y^{(u)}_i$ to denote the $i$-th entry of ${\boldsymbol{y}}^{(u)}$ (or in other words the $(u,i)$-th entry of ${\boldsymbol{Y}}$), and $w_{p,i}$ to denote the $(p,i)$-th entry of ${\boldsymbol{W}}$.

The MSE loss of TWIRLS defined in  \cref{eq:final-loss} can be decomposed into summation of sub-losses of each node, i.e.  \begin{align}\ell({\boldsymbol{W}})  = \sum_{u = 1}^n \frac{1}{2} \left\| {\boldsymbol{h}}_u^\top {\boldsymbol{X}} {\boldsymbol{W}} - {\boldsymbol{y}}^{(u)}\right\|_{\mathcal F}^2 = \sum_{u = 1}^n \ell^{(u)}\left({\boldsymbol{W}}\right),
\end{align}
where $\ell^{(u)}({\boldsymbol{W}}) = \frac{1}{2} \left\| {\boldsymbol{W}}^\top {\boldsymbol{X}} ^\top {\boldsymbol{h}}_u  - {\boldsymbol{y}}^{(u)}\right\|_{\mathcal F}^2$. For a specific class $i \in \{1,2,\cdots, c\}$, 
we have the gradient of $\ell^{(u)}$ w.r.t. ${\boldsymbol{w}}_i$ is \begin{align}
\frac{\partial \ell^{(u)}}{\partial {\boldsymbol{w}}_i} = {\boldsymbol{X}}^\top {\boldsymbol{h}}_u {\boldsymbol{h}}_u^\top {\boldsymbol{X}}{\boldsymbol{w}}_i - {\boldsymbol{X}}^\top {\boldsymbol{h}}_u 
 \times {{y}}^{(u)}_i.\label{eq:grad-mse}
\end{align}
It is not hard to see from \cref{eq:grad-mse} that the Hessian of $\ell^{(u)}$ with respect to ${\boldsymbol{w}}_i$ is \begin{align}
\nabla^2_{{\boldsymbol{w}}_i} \ell^{(u)}\left({\boldsymbol{W}}\right) = {\boldsymbol{X}}^\top {\boldsymbol{h}}_u {\boldsymbol{h}}_u ^\top {\boldsymbol{X}}. \label{eq:hessian-mse}
\end{align}

For classification tasks, the target ${\boldsymbol{y}}^{(u)}$-s are one-hot vectors representing the class of the node. Notice when we calculate cross entropy loss, we use $\mathsf{softmax}$ to normalize it before feeding it into the loss function. In the following for a vector ${\boldsymbol{v}} \in \mathbb R^c$ whose $i$-th entry is $v_i$, we define \begin{align}
\mathsf{softmax}({\boldsymbol{v}})_i = \frac{\exp(v_i)}{\sum_{j=1}^c \exp(v_j)}.
\end{align}

Suppose the $u$-th node belongs to class $k$, then the cross entropy loss of the $u$-th node is defined as \begin{align}
\mathsf{CE}^{(u)}({\boldsymbol{W}}) & = - \log \mathsf{softmax}\left({\boldsymbol{h}}_u^\top {\boldsymbol{X}} {\boldsymbol{W}}\right)_k
\\ & = -{\boldsymbol{h}}_u^\top  {\boldsymbol{X}}{\boldsymbol{w}}_k + \log \sum_{j=1}^c \exp\left({\boldsymbol{h}}_u^\top  {\boldsymbol{X}}{\boldsymbol{w}}_j\right),
\end{align}
The gradient of $\mathsf{CE}^{(u)}$ w.r.t. ${\boldsymbol{w}}_i$ is \begin{align}
\frac{\partial \mathsf{CE}^{(u)}}{\partial {\boldsymbol{w}}_i} & = -  {\boldsymbol{\delta}}_{i,k} \times {\boldsymbol{X}}^\top {\boldsymbol{h}_u} + \frac{{\boldsymbol{X}}^\top {\boldsymbol{h}}_u \exp\left({\boldsymbol{h}}_u^\top {\boldsymbol{X}} {\boldsymbol{w}}_i\right)}{\sum_{j=1}^c \exp\left({\boldsymbol{h}}_u^\top  {\boldsymbol{X}}{\boldsymbol{w}}_j\right)}
\\ & = {\boldsymbol{X}}^\top {\boldsymbol{h}}_u \mathrm{softmax} \left( {\boldsymbol{h}}_u^\top {\boldsymbol{X}} {\boldsymbol{W}} \right)_i - {\boldsymbol{X}}^\top {\boldsymbol{h}}_u \times {{y}}_i ^{(u)}.\label{eq:grad-ce}
\end{align}
By taking another partial differentiation, we have \begin{align}
\frac{\partial^2 \mathsf{CE}^{(u)}}{\partial w_{p,i} \partial w_{q,i}} & = \frac{ \left( {\boldsymbol{x}}_p^\top {\boldsymbol{h}}_u\right) \left( {\boldsymbol{x}}_q^\top {\boldsymbol{h}}_u\right) \exp\left({\boldsymbol{h}}_u^\top {\boldsymbol{X}} {\boldsymbol{w}}_i\right)}{\sum_{j=1}^c \exp\left({\boldsymbol{h}}_u^\top {\boldsymbol{X}} {\boldsymbol{w}}_j\right)} - \frac{ \left( {\boldsymbol{x}}_p^\top {\boldsymbol{h}}_u\right) \left( {\boldsymbol{x}}_q^\top {\boldsymbol{h}}_u\right) \left[ \exp\left({\boldsymbol{h}}_u^\top {\boldsymbol{X}} {\boldsymbol{w}}_i\right) \right]^2}{\left[\sum_{j=1}^c \exp\left({\boldsymbol{h}}_u^\top {\boldsymbol{X}} {\boldsymbol{w}}_j\right)\right]^2}
\\ & = \left( {\boldsymbol{x}}_p^\top {\boldsymbol{h}}_u\right) \left( {\boldsymbol{x}}_q^\top {\boldsymbol{h}}_u\right) \left[ s_i - s_i^2 \right], \label{eq:ce-hessian-scalar}
\end{align}
where $s_i  = \mathsf{softmax}\left( {\boldsymbol{h}}_u^\top {\boldsymbol{X}} {\boldsymbol{W}} \right)_i$. Rewriting \cref{eq:ce-hessian-scalar} into matrix form, we have \begin{align}
\nabla^2_{{\boldsymbol{w}}_i} \mathsf{CE}^{(u)}({\boldsymbol{W}}) = (s_i - s_i^2){\boldsymbol{X}}^\top {\boldsymbol{h}}_u {\boldsymbol{h}}_u^\top  {\boldsymbol{X}}. \label{eq:ce-hessian}
\end{align}

Comparing \cref{eq:grad-mse} and \cref{eq:grad-ce}, we can notice that the gradient of MSE loss and CE loss are essentially the same except the term ${\boldsymbol{h}}_u^\top {\boldsymbol{X}} {\boldsymbol{w}}_i$ in \cref{eq:grad-mse} is normalized by softmax in \cref{eq:grad-ce}. Additionally, by comparing the Hessian of MSE loss \cref{eq:hessian-mse} and the Hessian of CE loss \cref{eq:ce-hessian}, we have the latter is a rescaled version of the former. These similarities intuitively explains why our method can be also used on CE loss.

\subsection{The Unavoidable $\lambda^2$ in the Running Time}\label{sec:sqaure-tight}

\Cref{lem:square} essentially states the following idea: even when two PD matrices are very similar (in terms of spectral approximation), if they are ill-conditioned, their square can be very different. In this subsection, we prove that the bound obtained in \Cref{lem:square} is strict up to constant factors by explicitly constructing a worst case. Notice that, in this subsection we discuss the issue of squaring spectral approximators in a broader context, thus in this subsection we possibly overload some symbols used before to simplify the notation.

First we consider another definition of approximation rate: for two PD matrices ${\boldsymbol{\Sigma}}$ and $\widetilde {\boldsymbol{\Sigma}}$, we define the approximation rate as \begin{align}
\psi\left({\boldsymbol{\Sigma}}, \widetilde{\boldsymbol{\Sigma}}\right) = \max\left\{ \left\|  {\boldsymbol{\Sigma}}^{-1/2} \widetilde {\boldsymbol{\Sigma}}  {\boldsymbol{\Sigma}}^{-1/2} \right\|, \left\|  \widetilde{\boldsymbol{\Sigma}}^{-1/2}  {\boldsymbol{\Sigma}}  \widetilde{\boldsymbol{\Sigma}}^{-1/2} \right\|\right\}.
\end{align}
This definition is easier to calculate in the scenario considered in this subsection, and can be easily translated to the definition we used in \Cref{sec:preliminaries}: when $\epsilon = \psi\left({\boldsymbol{\Sigma}},\widetilde {\boldsymbol{\Sigma}}\right) \in (0,1)$, we have $\widetilde {\boldsymbol{\Sigma}} \approx_{\epsilon} {\boldsymbol{\Sigma}}$, and when $\psi\left({\boldsymbol{\Sigma}},\widetilde {\boldsymbol{\Sigma}}\right)$ is larger than $1$ we don't have an approximation in the form defined in \Cref{sec:preliminaries}.

Let $\gamma > 1$ and $\delta \in (0,1)$ be real numbers. Define 
\begin{align}{\boldsymbol{\Sigma}} = \begin{bmatrix}1 & 0\\ 0 & 1\end{bmatrix} \begin{bmatrix}\gamma & 0 \\ 0 & 1\end{bmatrix}\begin{bmatrix}1 & 0 \\ 0 & 1\end{bmatrix},\end{align}
and \begin{align}\widetilde {\boldsymbol{\Sigma}} = \begin{bmatrix}\sqrt{1-\delta^2} & -\delta\\ \delta & \sqrt{1-\delta^2}\end{bmatrix} \begin{bmatrix}\gamma & 0 \\ 0 & 1\end{bmatrix}\begin{bmatrix}\sqrt{1-\delta^2} & \delta\\ -\delta & \sqrt{1-\delta^2}\end{bmatrix}.\end{align}
It's not hard to compute the error rate $\psi\left({\boldsymbol{\Sigma}} , \widetilde{\boldsymbol{\Sigma}}\right) = \left\|  {\boldsymbol{\Sigma}}^{-1/2} \widetilde {\boldsymbol{\Sigma}}  {\boldsymbol{\Sigma}}^{-1/2} \right\| \approx \Theta (\gamma \delta^2)$. While if we consider the squared matrices, we have $\psi\left({\boldsymbol{\Sigma}}^2 , \widetilde{\boldsymbol{\Sigma}}^2\right) = \left\| {\boldsymbol{\Sigma}}^{-1} \widetilde {\boldsymbol{\Sigma}}^2  {\boldsymbol{\Sigma}}^{-1} \right\| \approx \gamma^2\delta^2$. In this case, $\psi\left({\boldsymbol{\Sigma}}^2 , \widetilde{\boldsymbol{\Sigma}}^2\right)$ is larger than $\psi\left({\boldsymbol{\Sigma}} , \widetilde{\boldsymbol{\Sigma}}\right)$ by a factor of $\gamma$, which is the condition number of ${\boldsymbol{\Sigma}}$. When $\gamma$ is very large and $\delta$ is very small, we can have a small $\psi\left({\boldsymbol{\Sigma}} , \widetilde{\boldsymbol{\Sigma}}\right)$ while large $\psi\left({\boldsymbol{\Sigma}}^2 , \widetilde{\boldsymbol{\Sigma}}^2\right)$, which .

\paragraph{A More General Construction.}

For a PD matrix ${\boldsymbol{\Sigma}}$ and its spectral approximation $\widetilde{\boldsymbol{\Sigma}}$, we call $\frac{\psi\left({\boldsymbol{\Sigma}}^2 , \widetilde{\boldsymbol{\Sigma}}^2\right)}{\psi\left({\boldsymbol{\Sigma}} , \widetilde{\boldsymbol{\Sigma}}\right)}$ the Squared Error Rate. As noted above, in the worst case the squared error rate can be as large as the condition number of ${\boldsymbol{\Sigma}}$. However, the construction above is limited to $2 \times 2$ matrices. Now we construct a more general worst case of the squared error rate and perform a loose analysis. Although not rigorously proved, the construction and the analysis suggest the origination of large squared error rates: it approaches the upper bound (condition number) when the eigenspace of the two matrices are very well aligned but not exactly the same.

Let ${\boldsymbol{A}}$ be an ill-conditioned matrix with all eigenvalues very large except one eigenvalue equals to $1$, and the eigenvalues of ${\boldsymbol{B}}$  are all closed to the eigenvalues of ${\boldsymbol{A}}$. Specifically, Let the SVD of ${\boldsymbol{A}}$ be ${\boldsymbol{A}} = {\boldsymbol{U}} \mathrm{diag}({\boldsymbol{\lambda}}) {\boldsymbol{U}}^\top$, where ${\boldsymbol{\lambda}} = \begin{bmatrix}\lambda_1 & \lambda_2 & \cdots \lambda_n\end{bmatrix}$, and suppose $\lambda_n = 1$ and $\lambda_k \gg 1, \forall k \leq n-1$. For simplicity we just let ${\boldsymbol{A}}$ and ${\boldsymbol{B}}$ have the same eigenvalues. Suppose the SVD of ${\boldsymbol{B}}$ is ${\boldsymbol{B}} = {\boldsymbol{V}} \mathrm{diag}({\boldsymbol{\lambda}}){\boldsymbol{V}}^\top$ and let ${\boldsymbol{\Lambda}} = \mathrm{diag}({\boldsymbol{\lambda}})$. We denote the one sided error rate of ${\boldsymbol{A}}$ and ${\boldsymbol{B}}$ by $\epsilon$, i.e. $\epsilon = \left\|{\boldsymbol{A}}^{-1/2} {\boldsymbol{B}} {\boldsymbol{A}}^{-1/2}\right\|$. We have   \begin{align}
\epsilon & = \left\| {\boldsymbol{\Lambda}}^{-1/2} {\boldsymbol{U}}^\top {\boldsymbol{V}} {\boldsymbol{\Lambda}} {\boldsymbol{V}}^\top {\boldsymbol{U}} {\boldsymbol{\Lambda}}^{-1/2} \right\|
\\ & = \left\| {\boldsymbol{\Lambda}}^{-1/2} {\boldsymbol{W}} {\boldsymbol{\Lambda}} {\boldsymbol{W}}^\top{\boldsymbol{\Lambda}}^{-1/2} \right\|,
\end{align}
where ${\boldsymbol{W}} = {\boldsymbol{U}}^\top {\boldsymbol{V}}$.

Since $\lambda_n = 1$ and for all $k \leq n-1$, $\lambda_k \gg 1$, we have \begin{align}{\boldsymbol{\Lambda}}^{-\frac{1}{2}} = \begin{bmatrix} \lambda_1^{-1/2} & & & \\ & \lambda_2^{-\frac{1}{2}} & & \\ & & \ddots &  \\& & & \lambda_n^{-\frac{1}{2}}\end{bmatrix} \approx \begin{bmatrix} 0 & & & \\ & 0 & & \\ & & \ddots &  \\& & & 1\end{bmatrix},\end{align}
and therefore \begin{align}
\epsilon & = \left\| {\boldsymbol{\Lambda}}^{-1/2} {\boldsymbol{W}} {\boldsymbol{\Lambda}} {\boldsymbol{W}}^\top{\boldsymbol{\Lambda}}^{-1/2} \right\|  \approx ({\boldsymbol{W}}{\boldsymbol{\Lambda}}{\boldsymbol{W}}^\top)_{n,n} = \sum_{k=1}^n {\boldsymbol{W}}_{k,n}^2 \lambda_k,
\end{align}
where ${\boldsymbol{A}}_{i,j}$ represents the $(i,j)$-th entry of a matrix ${\boldsymbol{A}}$. 

For simplicity, we assume all $\lambda_k (k \leq n-1)$ are approximate equal, i.e. $\lambda_1 \approx \lambda_2 \cdots \approx \lambda_{n-1} = \gamma$. Then we have \begin{align}\epsilon & = \sum_{k=1}^n {\boldsymbol{W}}_{k,n}^2 \lambda_k
\\ & \approx \gamma \left[ \sum_{k=1}^{n-1} {\boldsymbol{W}}_{k,n}^2 +  \frac{1}{\gamma}{\boldsymbol{W}}_{n,n}^2 \right]
\\ & \approx \gamma \sum_{k=1}^{n-1} {\boldsymbol{W}}_{k,n}^2
\\ & = \gamma\left(1 - {\boldsymbol{W}}_{n,n}^2\right).
\end{align}

Recall ${\boldsymbol{W}} = {\boldsymbol{U}}^\top {\boldsymbol{V}}$, we have ${\boldsymbol{W}}_{n,n} = {\boldsymbol{u}}_n^\top {\boldsymbol{v}}_n$, where ${\boldsymbol{u}}_n$ and ${\boldsymbol{v}}_n$ are the $n$-th column of ${\boldsymbol{U}}$ and ${\boldsymbol{V}}$ respectively. Thus we have \begin{equation}
\epsilon \approx \gamma\left[1 - \left({\boldsymbol{u}}_n^\top {\boldsymbol{v}}_n\right)^2\right].
\end{equation}

Now we can see the spectral approximation error $\epsilon$ is determined by two terms: the condition number $\gamma$ and the matchness of the eigenvectors corresponds to small eigenvalues, which is evaluated by $\left[1 - \left({\boldsymbol{u}}_n^\top {\boldsymbol{v}}_n\right)^2\right]$. \textbf{When $\gamma$ is very large but $\left[1 - \left({\boldsymbol{u}}_n^\top {\boldsymbol{v}}_n\right)^2\right]$ is small, ${\boldsymbol{B}}$ can still be a spectral approximation of ${\boldsymbol{A}}$}. For example if $(1-{\boldsymbol{u}}_n^\top {\boldsymbol{v}}_n) \approx \gamma^{-1}$, we can get a spectral approximation error $\epsilon \approx 1$. However, after we square the matrices, the eigenvalues will also get squared, but eigenvectors remains unchanged. That will enlarge  the spectral approximation error by a factor of  $\gamma$, i.e. $\left\|{\boldsymbol{A}}^{-1} {\boldsymbol{B}}^2 {\boldsymbol{A}}^{-1}\right\| \approx \gamma^2 \gamma^{-1} = \gamma$, which becomes very large. In this case the squared error rate is $\gamma$, the condition number of ${\boldsymbol{A}}$.

\subsection{Applying Graph Sparsification in Each Iteration}\label{sec:applying-graph-sparsification-in-each-iteration}

As mentioned in the main paper, unlike most existing work, in our algorithm we don't sparsify the graph in each training iteration. The proof of \Cref{lem:outer-analysis-naive} (see \Cref{sec:proof-outer-analysis}) suggests the reason why performing graph sparsification in each iteration can lead to suboptimal running time. In this subsection we illustrate this claim in detail. The intuition is that, in order to ensure convergence of the training, we require a small error rate in the gradient estimation, which is of the order $\epsilon^{1/2}$ as we have showed in \Cref{sec:proof-outer-analysis}. It is acceptable for the SDD solver because the running time of the SDD solver only logarithmly depends on the error rate. However, if we sparsify the graph, to obtain an $\epsilon^{1/2}$ error rate we will need to sample $O(n_\lambda / \epsilon)$ edges, which grows linearly with $1/\epsilon$, and can be large especially when $\epsilon$ is very small. Below is a more detailed analysis. 

Consider in the for-loop of \Cref{alg:herta} we replace the $\hat {\boldsymbol{L}}$ by a sparsified version ${\boldsymbol{L}}' = \mathsf{Sparsify}_{\omega}\left(\hat {\boldsymbol{L}}\right)$, where $\omega \in (0,1)$. Let ${\boldsymbol{H}}' = {\boldsymbol{I}} + \lambda {\boldsymbol{L}}'$. From \Cref{lem:regularized-spectral-sprsification}, we have ${\boldsymbol{H}}' \approx_{\omega} {\boldsymbol{H}}$. Consider $E_1$ defined in \Cref{sec:proof-outer-analysis}, it now becomes \begin{align}
E_1' & = \left\| {\boldsymbol{X}}^\top \left[\mathcal S\left({\boldsymbol{H}}'^{-1} {\boldsymbol{z}}\right) - {\boldsymbol{H}}^{-2} {\boldsymbol{z}}\right] \right\|
\\ & \leq \sigma_{\max}\left({\boldsymbol{X}}\right) \left\| \mathcal S\left({\boldsymbol{H}}'^{-1}{\boldsymbol{z}}\right) - {\boldsymbol{H}}^{-2}{\boldsymbol{z}}\right\|.
\end{align}
Here we can not proceed by using the fact that $\mathcal S$ is a linear solver, because now $\mathcal S$ is not a linear solver for ${\boldsymbol{H}}$ but for ${\boldsymbol{H}}'$, thus we will have to split the error term again: \begin{align}
E_1' \leq \sigma_{\max}\left({\boldsymbol{X}}\right) \left\| \mathcal S\left({\boldsymbol{H}}'^{-1} {\boldsymbol{z}}\right) - {\boldsymbol{H}}'^{-2} {\boldsymbol{z}}\right\| + \sigma_{\max}({\boldsymbol{X}}) \left\| {\boldsymbol{H}}^{-2} {\boldsymbol{z}} + {\boldsymbol{H}}'^{-2} {\boldsymbol{z}} \right\|.\label{eq:D.3.a}
\end{align}
Of the two terms on the right-hand side of \cref{eq:D.3.a}, the first one can be bounded with a similar method used in \Cref{sec:proof-outer-analysis}, and the second term is bounded by \begin{align}
\sigma_{\max}({\boldsymbol{X}}) \left\| {\boldsymbol{H}}^{-2} {\boldsymbol{z}} - {\boldsymbol{H}}'^{-2} {\boldsymbol{z}} \right\| & \leq \sigma_{\max}\left({\boldsymbol{X}}\right) \left\| I - {\boldsymbol{H}}^{-1}{\boldsymbol{H}}'^{-2}{\boldsymbol{H}}^{-1}\right\| \| {\boldsymbol{H}}^{-2} {\boldsymbol{z}}\|
\\ & \leq \sigma_{\max}\left({\boldsymbol{X}}\right)\omega \left\| {\boldsymbol{H}}^{-2} {\boldsymbol{z}} \right\|
\\ & \leq \sqrt\frac{8}{\epsilon} \kappa({\boldsymbol{X}}) \lambda \omega \left\| \nabla \ell_i\left({\boldsymbol{w}}\right) \right\|.
\end{align}
Therefore, in order to obtain $E_1' \leq \|\nabla \ell_i({\boldsymbol{w}})\|$, we at least require $\omega \leq \frac{\epsilon^{-1/2}}{\sqrt{8} \kappa({\boldsymbol{X}}) \lambda } = O\left(\epsilon ^{-1/2}\right)$. From \Cref{lem:regularized-spectral-sprsification}, the number of edges in the sparsified graph $\omega$ error rate is $O(n_{\lambda} / \epsilon)$. A similar analysis can be also applied to $E_2$, and by repeating the proof in \Cref{sec:proof-outer-analysis}, we obtain an overall running time bound \begin{align}
O(m + n_\lambda \lambda^2 d + d^3 + n_\lambda \epsilon^{-1} \log 1/\epsilon),
\end{align}
which, although eliminates the $m \left(\log \frac{1}{\epsilon}\right)^2$ term, introduces an extra $n_\lambda \epsilon^{-1} \log \frac{1}{\epsilon}$ term, and is usually worse than the original bound we derived in \Cref{thm:main}, especially when requiring a relatively small $\epsilon$. 

That being said, although not very likely, when $\epsilon$ and / or $\lambda$ are large, it is possible that directly sparsifying the graph in each iteration is beneficial. Considering this, we can adopt a mixed strategy: when $n_\lambda \epsilon^{-1} \leq m \log (1/\epsilon)$, we sparsify the graph in each iteration, elsewise we don't. This leads to the following overall running time bound: \begin{align}
O\left[m + n_\lambda \lambda^2 d + d^3 + \min\left\{ m \log\frac{1}{\epsilon} , n_\lambda \epsilon^{-1}\right\} \log\frac{1}{\epsilon}\right],
\end{align}
which is slightly better than the one we presented in \Cref{thm:main}.

\subsection{Possible Directions for Extending HERTA to More Complex Models}\label{sec:extending-to-more-complex}

The assumption that $f({\boldsymbol{X}};{\boldsymbol{W}}) =  {\boldsymbol{X}}{\boldsymbol{W}}$ provides us with a convenience that the Hessian of this model is a constant matrix. Therefore in \Cref{alg:herta} we only need to calculate the preconditioner ${\boldsymbol{P}}$ for one time. However, if $f({\boldsymbol{X}}; {\boldsymbol{W}})$ is implemented by a non-linear network, then the Hessian will change by the time and might be hard to calculate, which will be a key challenge to using a more complex $f$. We note that this can be possibly addressed by constructing a linear approximation of $f$ using its Jaccobian. In each iteration, we can use the Jaccobian to replace the ${\boldsymbol{X}}$ used in current version of HERTA, and recalculate ${\boldsymbol{P}}$ at each iteration. Since the convergence is fast, we only need to recalculate the Jaccobian for a small number of iterations, so should not bring massive change to the running time of the algorithm. 

The attention mechanism of TWIRLS in \cite{twirls} is achieved by adding a concave penalty function to each summand of the the $\mathrm{Tr} \left({\boldsymbol{Z}}^\top \hat {\boldsymbol{L}} {\boldsymbol{Z}}\right)$ term in \cref{eq:energy}. For specific penalty functions, it might still possible to find a inner problem solver, as long as the problem stay convex. We note that this depends on concrete implementation of the penalty term used. Investigating how to fast solve the inner problem under various penalty functions should also be an important problem for future study.

\section{Implementation Details}\label{sec:implementation-details}

In the experiments, the datasets are loaded and processed using the DGL package \cite{dgl}. We use the original inner problem solver in \cite{twirls} since it is not computation bottleneck. It can be in principle replaced by any implementation of SDD solvers. 

For the calculating the gradient of the preconditioned model, we presented a calculation method in \Cref{alg:herta} which maintains the lowest computational complexity. In preliminary experiments, we tested this calculation method with using the autograd module in pytorch\footnote{https://pytorch.org/} and verified that they have the same output, and similar computational efficiency on real world datasets (again in practice this step is not a bottleneck). Therefore, we simply use pytorch autograd module to compute gradients in experiments.

Using the pytorch autograd module also enables us to apply HERTA on various loss functions and optimizers: we only need to perform the preconditioning and indicate the loss function. The gradient and optimization algorithm with be automatically realized by pytorch.

In order to ensure a fair comparison and prevent confounding factors, we don't using any common training regularization techniques like weight decay or dropout. For each setting, we repeat the experiment with learning rates in $\{0.001,0.01,0.1,1,10\}$ and choose the trial which the training loss does not explode and with the lowest final training loss to report in the main paper.

\begin{figure*}[tbp]
\begin{minipage}{0.33\linewidth}
    \centering
    \includegraphics[width=\linewidth]{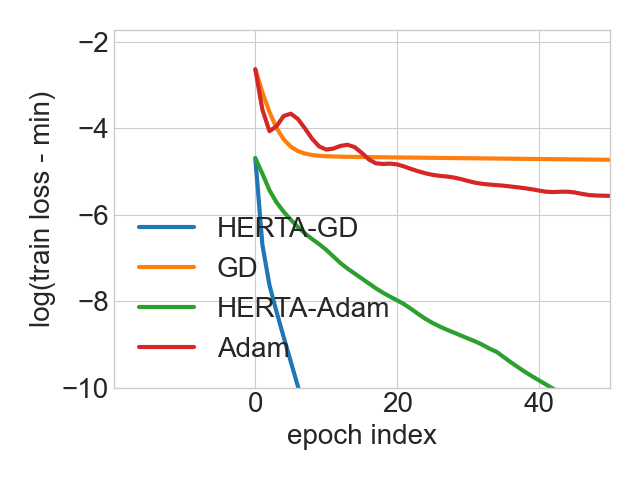}
\end{minipage}
\begin{minipage}{0.33\linewidth}
    \centering
    \includegraphics[width=\linewidth]{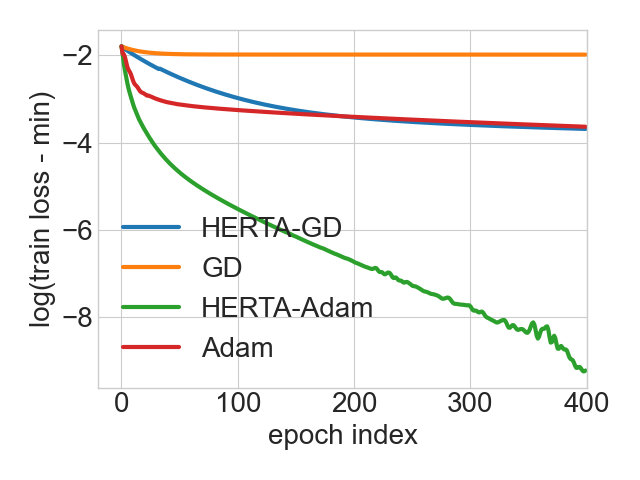}
\end{minipage}
\begin{minipage}{0.33\linewidth}
    \centering
    \includegraphics[width=\linewidth]{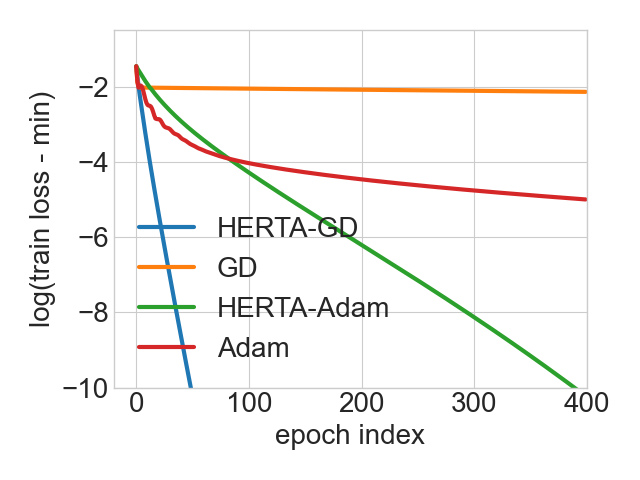}
\end{minipage}
    \caption{The training loss comparison between HERTA and standard optimizers on MSE loss with $\lambda = 20$. Dataset used from left to right: ogbn-arxiv, citeseer, pubmed. }
    \label{fig:ours-mse-20}
\end{figure*}

\begin{figure*}[tbp]
\begin{minipage}{0.33\linewidth}
    \centering
    \includegraphics[width=\linewidth]{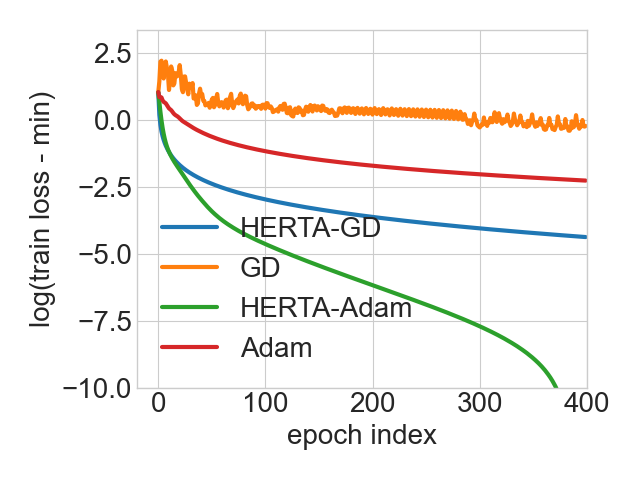}
\end{minipage}
\begin{minipage}{0.33\linewidth}
    \centering
    \includegraphics[width=\linewidth]{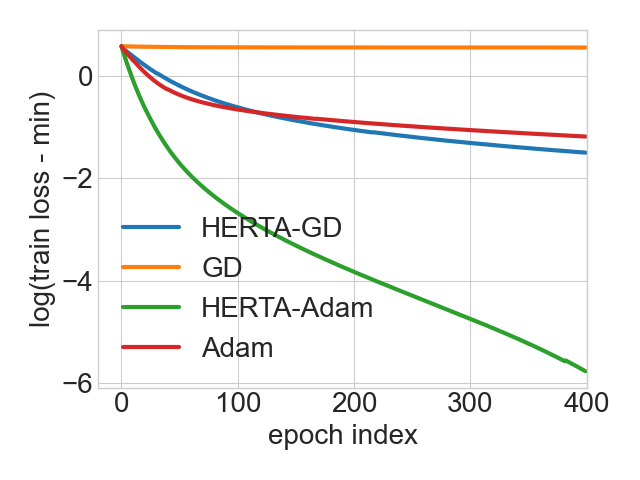}
\end{minipage}
\begin{minipage}{0.33\linewidth}
    \centering
    \includegraphics[width=\linewidth]{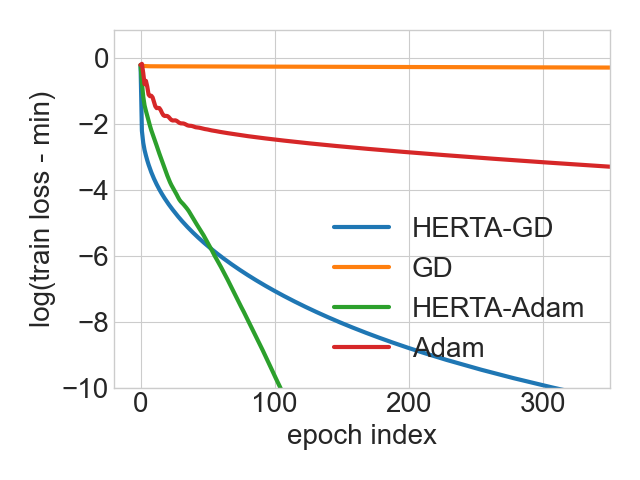}
\end{minipage}
    \caption{The training loss comparison between HERTA and standard optimizers on CE loss with $\lambda = 20$. Dataset used from left to right: ogbn-arxiv, citeseer, pubmed. }
    \label{fig:fastce-ce-20}
\end{figure*}

\begin{figure*}[htbp]
\hfill
\begin{minipage}{0.35\linewidth}
    \centering
    \includegraphics[width=\linewidth]{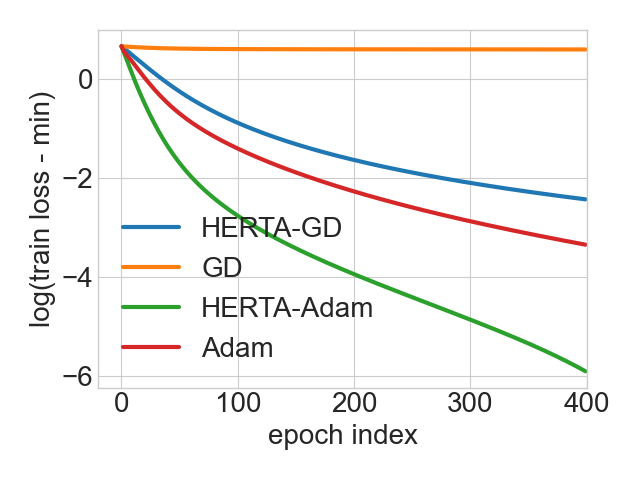}
\end{minipage}
\hfill
\begin{minipage}{0.35\linewidth}
    \centering
    \includegraphics[width=\linewidth]{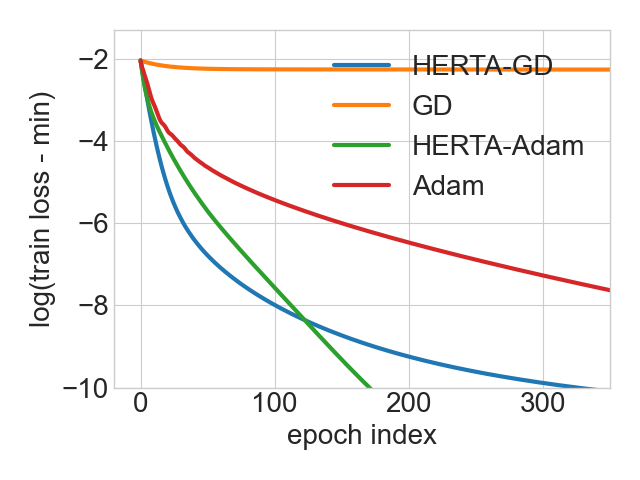}
\end{minipage}
\hfill
    \caption{The training loss comparison between HERTA and standard optimizers on Cora with $\lambda = 1$. Left: CE loss. Right: MSE loss.}
    \label{fig:cora-1}
\end{figure*}

\section{Additional Experiment Results}\label{sec:appendix-experiment}

In this section, we present additional experiment results.

\subsection{Experiments with Larger $\lambda$}

All the experiments presented in the main paper are with $\lambda = 1$. In this subsection, we present results with $\lambda = 20$. See \Cref{fig:ours-mse-20,fig:fastce-ce-20} for results with MSE loss and CE loss respectively. The results supports our observation in the main paper that HERTA works consistently well on all settings.

\subsection{Additional Results on Cora}

In this subsection we present results on Cora, another citation dataset used in \cite{twirls}. See \Cref{fig:cora-1} and \Cref{fig:cora-20} for results with $\lambda = 1$ and $\lambda = 20$ respectively. It is clear that the results on Cora is consistent with our observation on other datasets.

\begin{figure*}[htbp]
\hfill
\begin{minipage}{0.35\linewidth}
    \centering
    \includegraphics[width=\linewidth]{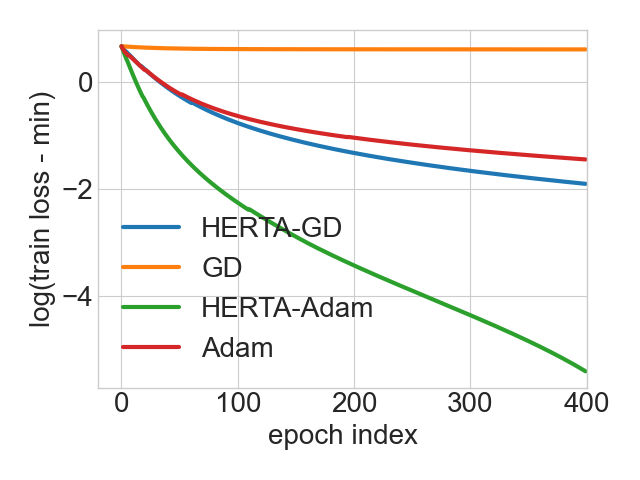}
\end{minipage}
\hfill
\begin{minipage}{0.35\linewidth}
    \centering
    \includegraphics[width=\linewidth]{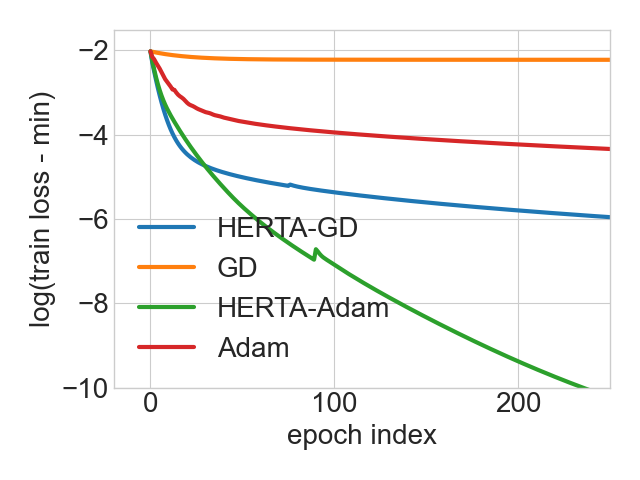}
\end{minipage}
\hfill
    \caption{The training loss comparison between HERTA and standard optimizers on Cora with $\lambda = 20$. Left: CE loss. Right: MSE loss.}
    \label{fig:cora-20}
\end{figure*}

\end{document}